\documentclass{article}

\usepackage{microtype}
\usepackage{graphicx}
\usepackage{subfigure}
\usepackage{booktabs} % for professional tables

\usepackage{hyperref}

\usepackage{amsmath}
\usepackage{amssymb}
\usepackage{mathtools}
\usepackage{dsfont}
\usepackage{bm}
\usepackage{cleveref}
\usepackage{amsthm}
\usepackage{graphicx}

\newcommand{\lr}[1]{\left(#1\right)}
\newcommand{\lrs}[1]{\left[#1\right]}
\newcommand{\lrscal}[1]{\left\langle#1\right\rangle}
\newcommand{\lrc}[1]{\left\{#1\right\}}
\newcommand{\R}{\mathcal R}

\newcommand{\prob}[1][]{\ifthenelse{\isempty{#1}{}}{\mathbb P}{\mathbb P \left[#1\right]}}
\newcommand{\E}[1][]{\mathbb E \left[#1\right]}

\DeclareMathOperator*{\argmax}{arg\,max}
\DeclareMathOperator*{\argmin}{arg\,min}
\renewcommand{\Pr}{\mathbb{P}}

\newcommand{\Ind}[1]{\mathds{1}{\lr{{#1}}}}

\DeclarePairedDelimiter{\floor}{\lfloor}{\rfloor}
\newcommand\numberthis{\addtocounter{equation}{1}\tag{\theequation}}

\newcommand{\scO}{\mathcal{O}}

\newcommand{\algoname}[1]{\ifmmode\operatorname{#1}\else{\normalfont #1}\fi}
\newcommand{\TsallisB}{\algoname{Tsallis-Switch}}

\newtheorem{theorem}{Theorem}
\newtheorem{lemma}[theorem]{Lemma}

\newtheorem{corollary}[theorem]{Corollary}
\usepackage[accepted]{4arXiv}

\icmltitlerunning{An Algorithm for Stochastic and Adversarial Bandits with Switching Costs}

\begin{document}

\twocolumn[
\icmltitle{An Algorithm for Stochastic and Adversarial Bandits with Switching Costs}



\icmlsetsymbol{equal}{*}

\begin{icmlauthorlist}
\icmlauthor{Chlo{\'e} Rouyer}{dk}
\icmlauthor{Yevgeny Seldin}{dk}
\icmlauthor{Nicol{\`o} Cesa-Bianchi}{it}
\end{icmlauthorlist}

\icmlaffiliation{dk}{Department of Computer Science, University of Copenhagen, Denmark}
\icmlaffiliation{it}{DSRC \& Dept. of Computer Science, Universit{\`a} degli Studi di Milano, Milano, Italy}

\icmlcorrespondingauthor{Chlo{\'e} Rouyer}{chloe@di.ku.dk}

\icmlkeywords{Machine Learning, ICML, Bandits, Switching Cost}

\vskip 0.3in
]



\printAffiliationsAndNotice{}  % leave blank if no need to mention equal contribution

\begin{abstract}
%% Supposeldy 4-6 sentences long
% We consider the problem of bandits with switching costs. This problem is a variation of the classic multi-armed bandits problem for which the learner receives a penalty $\lambda$ that adds to the regret every time that she changes the arm she plays. We present an algorithm that achieves an optimal pseudo-regret bound in $O(\lambda^{1/3}K^{1/3}T^{2/3})$ in the oblivious adversarial regime, which matches the $\Tilde\Omega(\lambda^{1/3}K^{1/3}T^{2/3})$ lower bound derived by \citet{DDKP13}. In the stochastically constrained adversarial regime, we simultaneously derive a refined pseudo regret bound scaling as $ O\lr{\sum_{i \neq i*}  \lambda^{2/3} K^{2/3} T^{1/3}}$. We also explore different variations of the problem where we allow the switching costs to vary over time.
We propose an algorithm for stochastic and adversarial multiarmed bandits with switching costs, where the algorithm pays a price $\lambda$ every time it switches the arm being played. Our algorithm is based on adaptation of the Tsallis-INF algorithm of \citet{ZS21} and requires no prior knowledge of the regime or time horizon. 
In the oblivious adversarial setting it achieves the minimax optimal regret bound of $\scO\big((\lambda K)^{1/3}T^{2/3} + \sqrt{KT}\big)$, where $T$ is the time horizon and $K$ is the number of arms. In the stochastically constrained adversarial regime, which includes the stochastic regime as a special case, it achieves a regret bound of $\scO\left(\big((\lambda K)^{2/3} T^{1/3} + \ln T\big)\sum_{i \neq i^*} \Delta_i^{-1}\right)$, where $\Delta_i$ are the suboptimality gaps and $i^*$ is a unique optimal arm. In the special case of $\lambda = 0$ (no switching costs), both bounds are minimax optimal within constants. We also explore variants of the problem, where switching cost is allowed to change over time. We provide experimental evaluation showing competitiveness of our algorithm with the relevant baselines in the stochastic, stochastically constrained adversarial, and adversarial regimes with fixed switching cost.
\end{abstract}

\section{Introduction}
\label{sec:intro}
%  consider the problem of multi-armed bandits with switching costs. In this problem, the learner has not only to suffer the loss of the actions she took, but each time she plays an action that is different from the one played at the previous round, a penalty of $\lambda$ is added to the regret. 

% In the adversarial regime, \citet{DDKP13} have derived a lower bound for $RS(T, \lambda)$ which scales as $\Tilde\Omega(\lambda^{1/3}K^{1/3}T^{2/3})$, and \citet{ADK12} derived an upper bound for $RS(T, 1)$ scaling as $\Tilde O(K^{1/3}T^{2/3})$. Our goal in this regime is to achieve a bound scaling as $O(\lambda^{1/3}K^{1/3}T^{2/3})$ in the adversarial regime, and refined guarantees in the stochastically constrained adversarial regime, as well as in intermediate regimes.

Multiarmed bandits are the reference framework for the study of a wide range of sequential decision-making problems, including recommendation, dynamic content optimization, digital auctions, clinical trials, and more. In many application domains, algorithms have to pay an additional penalty $\lambda > 0$ each time they play an arm different from the one played in the previous round. Such switching cost may occur in the form of a transaction cost in financial trading, or a reconfiguration cost in industrial environments.

So far, the problem of bandits with switching costs has been studied using algorithms whose optimality depends on the nature of the source of losses (or, equivalently, rewards) for the $K$ arms. In the oblivious adversarial case, when losses are generated by an arbitrary deterministic source, \citet{DekelTA12} use a simple variant of the Exp3 algorithm to prove an upper bound of $\scO\big((K\ln K)^{1/3}T^{2/3}\big)$ for $\lambda=1$ (i.e., unit switching cost), where $T$ is the time horizon ---see also \citep{blum2007learning} for an earlier, slightly weaker result. A result by \citet{DDKP13} implies a lower bound of $\Omega\big((\lambda K)^{1/3}T^{2/3} + \sqrt{KT}\big)$ for all $\lambda \ge 0$. Note the sharp transition: if $\lambda > 0$, then the regret asymptotically grows as $T^{2/3}$ as opposed to $\sqrt{T}$ when there is no switching cost.

In the stochastic case, where losses for each arm are generated by an i.i.d.\ process, \citet{NIPS2019_8341} and \citet{esfandiari2019regret} use arm elimination algorithms to prove that $\scO(\ln T)$ switches are sufficient to achieve the optimal distribution-dependent regret of $\scO\big((\ln T)\sum_{i\,:\, \Delta_i > 0} \Delta_i^{-1}\big)$, where $\Delta_i$ is the suboptimality gap of arm $i$. Hence, in the stochastic case the introduction of switching costs does not lead to a qualitative change of the minimax regret rate.

In practical applications, it is desirable to have algorithms that require no prior knowledge about the nature of the loss generation process, maintain robustness in the adversarial regime, and have the ability to achieve lower regret in the stochastic case. A number of such algorithms have been developed for the standard multiarmed bandits \citep{BS12,SS14,AC16,SL17,WL18,ZS19,ZS21} and the ideas have been extended to several other domains, including combinatorial bandits \citep{ZLW19}, decoupled exploration and exploitation \citep{RS20}, and episodic MDPs \citep{JL20}. We aim at designing algorithms with similar properties for bandits with switching costs.

\subsection*{Main contributions}
Our starting point is the Tsallis-INF algorithm of \citet{ZS21}, which was shown to achieve minimax regret rates in both stochastic and adversarial regimes for standard bandits. We introduce a modification of this algorithm, which we call \TsallisB, to take care of the switching costs. In the adversarial regime, the regret bound of \TsallisB\  matches (within constants) the minimax optimal regret bound $\Theta\big((\lambda K)^{1/3}T^{2/3} + \sqrt{KT}\big)$ for any value of $\lambda \ge 0$. In the stochastically constrained adversarial regime, which includes the stochastic regime as a special case, we prove a bound $\scO\left(\big((\lambda K)^{2/3} T^{1/3} + \ln T\big)\sum_{i \neq i^*} \Delta_i^{-1}\right)$, where $i^*$ is a unique optimal arm. Note that, in the special case of $\lambda = 0$ (no switching costs), we recover (up to constant factors) the minimax optimal bounds of Tsallis-INF for both regimes. Similarly to Tsallis-INF, our algorithm is fully oblivious to both the regime and the time horizon $T$.

\TsallisB, which runs Tsallis-INF as a subroutine, uses the standard tool to control the frequency of arm switching: game rounds are grouped into consecutive blocks $B_1,B_2,\ldots$, and \TsallisB\ runs Tsallis-INF over the blocks, preventing it from switching arms within each block. The number of switches is thus bounded by the number of blocks. Since $T$ is unknown, we use block sizes of increasing length. As a new arm is drawn only at the beginning of each block, the effective range of the losses experienced by Tsallis-INF grows with time. Therefore, we modify the analysis of Tsallis-INF to accommodate losses of varying range. This extension may potentially be of independent interest.

\section{Problem Setting and Notations}

We consider a repeated game with $K$ arms and a switching cost $\lambda \ge 0$. At each round $t = 1, 2, \dots$ of the game, the environment picks a loss vector $\ell_t \in [0, 1]^K$, and the algorithm chooses an arm $J_t \in [K]$ to play. The learner then incurs the loss $\ell_{t, J_t}$, which is observed. If $J_t \neq J_{t-1}$, then the learner also suffers an extra penalty of $\lambda$. We use the same setting as \citet{DDKP13}, and assume that $J_0 = 0$, which means that there is always a switch at the first round.

We consider two regimes for the losses. In the oblivious adversarial regime, the loss vectors $\ell_t$ are arbitrarily generated by the environment and do not depend on the actions taken by the learner.
We also work in the stochastically constrained adversarial regime. This setting, introduced by \citet{WL18}, generalizes the widely studied stochastic regime by allowing losses to be drawn from distributions with fixed gaps. This means that at for all $i$, $\E[\ell_{t,i}]$ can fluctuate with $t$, but for all $i, j$,  $\E[\ell_{t,i} - \ell_{t,j}] = \Delta_{i,j}$ remains constant. The suboptimality gaps are then defined as $\displaystyle \Delta_i = \Delta_{i,1} - \min_j \Delta_{j,1}$.
%% add bandits with corruption
% We generalize this regime even further by considering the regime of stochastic bandits with adversarial corruption. 

We define the pseudo-regret with switching costs as follows,
\begin{align*}
     \text{RS}(T, \lambda)
&=
    \E[\sum_{t = 1}^T \ell_{t, J_t}] -  \min_i\E[ \sum_{t = 1}^T\ell_{t, i}]
\\&
     + \lambda \sum_{t = 1}^T \Pr(J_{t - 1} \neq J_t)\label{RST}
\\&=
    R_T + \lambda\,S_T, \numberthis
\end{align*}
where $J_t$ is the action played by the learner at round $t$, and $\lambda$ is the switching cost.
We recognize that the $R_T = \text{RS}(T, 0)$ is the classical definition of the pseudo regret (without switching costs), while $S_T$ counts the expected number of switches.
%\begin{equation}
%    \text{R}_T= \E[\sum_{t = 1}^T \ell_{t, J_t}] - \min_i\E[ \sum_{t = 1}^T\ell_{t, i}]  \label{RT}. 
%\end{equation}
Furthermore, we recall that in the stochastically constrained adversarial regime, the pseudo-regret can be rewritten in terms of the sub-optimality gaps, as:
\begin{equation}
    \text{R}_T= \sum_{t = 1}^T \sum_{i=1}^K \E[p_{t, i}]\Delta_i\label{RTstoc},
\end{equation}
where $p_{t, i}$ is the probability of playing action $i$ at round $t$. 
%For clarity, let's define $S_T = \sum_{t = 1}^T \EEye[J_{t - 1} \neq J_t]$, and thus we have:
%\[ RS(T, \lambda) = R_T + \lambda S_T. \]

\section{Working with Blocks}
\label{sec:blocks}
In order to control $S_T$, we limit the number of switches between actions that the algorithm makes by dividing game rounds into blocks and forcing the algorithm to play the same action for all the rounds within a block.
%Let $|B_n|$ be the length of the $n$-{th} block $B_n$ and let $N$ be the smallest number of blocks required to cover $T$ rounds.
Given a sequence of blocks $\lr{B_n}_{n\ge 1}$ of lengths $|B_n|$, and a time horizon $T$, we define $N$ as the smallest integer such that $\sum_{n=1}^N|B_n| \geq T$, and we truncate the last block such that $N$ blocks sum up to $T$.

As $S_T \leq N$, we just need to bound $N$ and the pseudo-regret $R_T$ (without the switching costs) over the $N$ blocks. Let $c_{n, i} = \sum_{s \in B_n} \ell_{s, i}$ be the cumulative loss of playing action $i$ in block $n$. Since $\ell_{t, i} \in [0, 1]$, we have $c_{n, i} \in \big[0, |B_n|\big]$. We use $I_n$ to refer to the action played by the algorithm in block $n$.
% For all rounds that are in block $n$, the algorithm plays $I_n$,
Then, for all $t \in  B_n$, we have $J_t = I_n$ and
\[ \text{R}_T = \E[\sum_{n = 1}^{N} c_{n, I_n}] - \min_j\E[ \sum_{n = 1}^{N} c_{n, j}]. \]

\begin{algorithm}[tb]
	\caption{\TsallisB}
	\label{alg:Tsallis_blocks}
\begin{algorithmic}
	\STATE {\bfseries Input: }{Learning rates $\eta_1 \geq \eta_2 \geq \dots > 0$. \\Block lengths $|B_1|, |B_2|, \dots$. }\\
	\STATE {\bfseries Initialize: }{$\bm \tilde C_0 = \bm0_K$}\\
	\FOR{$n = 1, 2, \dots$}
		\STATE ${\displaystyle p_n = \argmin_{p \in \Delta^{K-1}} \left\{ \big\langle p, \tilde C_{n-1} \big\rangle  -  \sum_{i = 1}^K \frac{4\sqrt{p_i} - 2 p_i}{\eta_n} \right\} }$ \\
		\STATE Sample $I_n \sim p_n$ and play it for all rounds $t \in B_n$ \\
		\STATE Observe and suffer $c_{n,I_n} = \sum_{t \in B_n} \ell_{t, I_n}$. \\
		\STATE ${\displaystyle \forall i \in [K]: \tilde c_{n, i}
%		= \frac{ c_{n, i} \Ind{I_n = i}}{p_{n, i}}
= \begin{cases} \frac{c_{n, i}}{p_{n, i}}, &\mbox{if } I_n = i, \\
		0, & \mbox{otherwise.} \end{cases} }$ \\
		\STATE$\forall\  i \in [K]: \quad \tilde C_n(i) = \tilde C_{n-1}(i) + \tilde c_{n, i}$.
		\ENDFOR
\end{algorithmic}
\end{algorithm}

\section{The Algorithm}
\label{sec:algo}
Our \TsallisB\ algorithm (see Algorithm~\ref{alg:Tsallis_blocks}) calls Tsallis-INF at the beginning of each block to obtain an action, plays the proposed action in each step of the block, and then feeds back to Tsallis-INF the total loss suffered by the action over the block. As blocks have varying lengths, we need to adapt the analysis of Tsallis-INF to losses of varying range.
% We adapt the Tsallis-Inf algorithm of \citet{ZS21} to take care of the varying range of the losses. We then call the adapted Tsallis-Inf algorithm at the beginning of each block, play the proposed action throughout the duration of the block, and feed the cumulative loss back to adapted Tsallis-INF, see .

\section{Main Results}
% The main contribution of this work is the following. We consider the problem of bandits with switching costs, where the switching cost $\lambda$ is a fixed parameter given to the algorithm. Unlike the time horizon $T$, $\lambda$ is known in advance and can be used to tune the block lengths.
We start by considering the case where the switching cost $\lambda$ is a fixed parameter given to the algorithm. Since $\lambda$ is known in advance, it can be used to tune the block lengths.
\begin{theorem}
\label{th_param_switches}
Let $\lambda \geq 0$ be the switching cost. Define blocks with lengths $|B_n| = \max\lrc{\lceil a_n \rceil, 1}$, where $a_n = \frac{3\lambda}{2}\sqrt{\frac{n}{K}}$. The preudo-regret of \TsallisB\ with learning rate $\eta_n = \frac{2}{a_n +1}\sqrt{\frac{2}{n}}$ executed over the blocks in any adversarial environment satisfies:
 \begin{align*} 
 R(T, \lambda) &\leq  5.25 (\lambda K)^{1/3}T^{2/3} + 6.4 \sqrt{KT}\\
 &\qquad+ 3\sqrt{2K}+ 5.25\lambda + 6.25.  
\end{align*}
Furthermore, in any stochastically constrained adversarial regime with a unique best arm $i^*$, the pseudo-regret additionally satisfies:
\begin{align*}
    R&(T, \lambda)
\leq
    \lr{66 (\lambda K)^{2/3}T^{1/3} + 32 \ln T}\sum_{i \neq i^*} \frac{1}{\Delta_i}
\\&+
    \lr{ 160\lambda^{2/3}T^{1/3}K^{1/6}  + 160\lambda + 49 \lambda^2 + 32} \sum_{i \neq i^*} \frac{1}{\Delta_i}
\\&+
    \frac{544\lambda}{\sqrt{K}} + \lambda + 66.
\end{align*}
\end{theorem}
A proof is provided in Section~\ref{sec:proofs}. For $\lambda=0$ (no switching costs) both regret bounds match within constants the corresponding bounds of Tsallis-INF for multiarmed bandits with no switching costs. Furthermore, in the adversarial regime the algorithm achieves the optimal regret rate for all values of $\lambda$. In the stochastically constrained adversarial regime, for $\lambda > 0$ the regret grows as $T^{1/3}$ rather than logarithmically in $T$. This is also the case for the stochastic regime, which is a special case. While the algorithm does not achieve the logarithmic regret rate in the stochastic regime, as do the algorithms of 
\citet{NIPS2019_8341} and \citet{esfandiari2019regret}, it still exploits the simplicity of the regime and reduces the regret rate from $T^{2/3}$ to $T^{1/3}$. Additionally, in contrast to the work of \citet{NIPS2019_8341} and \citet{esfandiari2019regret}, the stochastic regret guarantee holds simultaneously with the adversarial regret guarantee, and the algorithm requires no knowledge of the time horizon. We also note that we are unaware of specialized lower bounds for the more general stochastically constrained adversarial regime with switching costs, and it is unknown whether the corresponding regret guarantee is minimax optimal. 
%In the stochastic setting, we use the self-bounding technique and a refined bound on the number of switches in order to limit the impact of the number of switches on the final bound. 
Theorem~\ref{th_param_switches} is based on the following generalized analysis of the Tsallis-INF algorithm to accommodate losses in varying intervals. The result may be of independent interest. 
\begin{theorem}
\label{th_wo_switching}
Consider a multi-armed bandit problem where the loss vector at round $t$ belongs to $[0, b_t]^K$ and $b_t$ is revealed to the algorithm before round $t$.
Then the pseudo-regret of \TsallisB\ in any adversarial environment for any positive and non-decreasing sequence of learning rates $\lr{\eta_t}_{t\ge 1}$ satisfies
\begin{equation}
\label{eq:wo-adv}
    R_T
\le
    \sqrt{K}\left(\sum_{t = 1}^T \frac{\eta_t}{2}b_t^2 + \frac{4}{\eta_T}\right) + 1.
\end{equation}
Furthermore, in the stochastically constrained adversarial regime with a unique best arm $i^*$, the pseudo regret also satisfies 
\begin{equation}
\label{eq:wo-stoch}
    R_T
\leq \sum_{t = 1}^{T} \sum_{i \neq i^*} \frac{\lr{\frac{7}{2}\eta_tb_t^2 + 2c\lr{\eta_t^{-1} - \eta_{t-1}^{-1}}}^2}{4\Delta_i b_t} + \sum_{t = 1}^{T_0}  \eta_t b_t^2  + 2,
\end{equation} 
where  $c = \begin{cases} 2 \ \mbox{ if } \frac{5\eta_t}{4}b_t^2 \geq 2 \lr{\eta_t^{-1} - \eta_{t-1}^{-1}},\\4 \ \mbox{otherwise}  \end{cases}$.
\end{theorem}
In particular, if $b_t = B$ for all rounds $t$, we have the following more interpretable result.
\begin{corollary} 
\label{coro_fixed_B}
Consider a multi-armed bandit problem with loss vectors belonging to $[0,B]^K$.
Then the pseudo-regret of Tsallis-INF with $\eta_t = \frac{2}{B\sqrt{t}}$ satisfies
$R_T \leq 4B\sqrt{KT} + 1$ in any adversarial regime.
Furthermore, in the stochastically constrained adversarial regime with a unique best arm $i^*$, the pseudo regret additionally satisfies
\begin{align*}
    R_T \leq &   21 B(\ln T +1)\sum_{i \neq i^*} \frac{1}{\Delta_i} + 8\sqrt{B}  + 2.
\end{align*}
\end{corollary}

\subsection{Varying Switching Cost}
Now we consider a setting where
%each time you make a switch, the switching cost changes.
the switching cost may change after each switch. The learner is given the $n$-th switching cost $\lambda_n$ at the beginning of block $B_n$, and we allow the length of the block $|B_n|$ to depend on it.
In this setting, the cumulative switching cost becomes 
\[S\lr{T, \lr{\lambda_n}_{n\ge 1}} = \sum_{n = 1}^{N} \lambda_n \Pr(I_n \neq I_{n-1}), \]
where, as before, $N$ is the smallest number of blocks to cover $T$ rounds.
 We construct blocks such that the contribution of the terms $R_T$ and $S\lr{T, \lr{\lambda_n}_{n\ge 1}}$ remains balanced.
\begin{theorem}
\label{th_block_lambda_n}
Let $(\lambda_n)_{n\ge 1}$ be a sequence of non-negative switching costs. 
The pseudo-regret with switching costs of \TsallisB\ run with block lengths $|B_n| = \max\big\{\big\lceil \sqrt{{\lambda_n a_n}/{K}}\,\big\rceil , 1\big\}$ and $\eta_n = \frac{2\sqrt{2K}}{3a_n}$, where $a_n = \big(\sum_{s = 1}^n \lambda_s + \sqrt{{K}/{s}}\big)$, satisfies:
\begin{equation}
\label{eq:var-adv}
    R(T, \lambda)\leq \sum_{n =1}^{N} 7 \lambda_n + 12 \sqrt{KN} + 2 ,
\end{equation}
where $N$ is the smallest integer such that $\sum_{n = 1}^{N} |B_n| \geq T$.
Furthermore, in the stochastically constrained adversarial regime with a unique best arm $i^*$, the pseudo regret additionally satisfies 
\begin{align*}
R\lr{T, \lr{\lambda_n}_{n\ge 1}} 
   & \leq  \sum_{n = 1}^N\sum_{i \neq i^*}  \frac{ \lr{11 \lambda_n  + \lambda_{n + 1} + \frac{10 \sqrt{2}}{\sqrt{n}}}^2}{4\Delta_i |B_n|} \\
   &~~ + \sum_{n = 1}^{N_0} \lr{\frac{2\sqrt{2}\lambda_n}{\sqrt K}}
+ 4\sqrt{2N_0} + \lambda_1 + 2,
\end{align*}
where $N_0$ is the smallest $n \leq N$ such that for all $ n \geq N_0$, $\eta_n |B_n| \leq \frac{1}{4}$. If such an integer does not exist, then $N_0 = N$.
\end{theorem}
A proof is provided in Appendix~\ref{appen_varying_switching_cost}.
Note that for $\lambda_n = \lambda$, bound \eqref{eq:var-adv} for the adversarial setting is of the same order as the corresponding bound in Theorem~\ref{th_param_switches}.

If $\lambda_n$ is not monotone, then controlling the first term in the above regret bound is challenging, because the block length $|B_n|$ in the denominator does not depend on $\lambda_{n+1}$ in the numerator. Below, we provide a specialization of the regret bound
% under some assumptions on the behaviour of $\lambda_n$. The first one assumes that the switching costs are non-increasing, and the second one
assuming that the switching costs increase as $\lambda_n = n^\alpha$ for some $\alpha > 0$. Proof is provided in Appendix~\ref{appen_varying_switching_cost}.
% \begin{corollary} 
% \label{cor_decreasing}
% Assume that  $(\lambda_n)_{n\ge 1}$ is a non-increasing sequence. Then the regret bound for stochastically constrained adversarial regime with a unique best arm $i^*$ in Theorem~\ref{th_block_lambda_n} satisfies
% \begin{align*}
%   R\lr{T, \lr{\lambda_n}_{n\ge 1}} 
%  & \leq \sum_{i \neq i^*} \sum_{n = 1}^N\ \lr{36\frac{\lambda_n \sqrt{K}}{ \sqrt{n}\Delta_i} +  85\frac{\lambda_n}{\sqrt n\Delta_i }} \\
% &\quad+ \sum_{i \neq i^*}\lr{ \frac{50 \log N}{\Delta_i} + \frac{50}{\Delta_i} +  \frac{ 36 \lambda_N^2}{\Delta_i}} 
% \\ &\quad +\sum_{n = 1}^{171} \lr{\frac{2\sqrt{2}\lambda_n}{\sqrt K}}
%  + \lambda_1 + 76,
% \end{align*}
% \end{corollary}
% We observe that in particular, if the switching costs are constant, we recover a bound that scales as $\scO\left(\big((\lambda K)^{2/3} T^{1/3} + \ln T\big)\sum_{i \neq i^*} \Delta_i^{-1}\right)$, which is the same rate as Theorem \ref{th_param_switches}.
%
\begin{corollary}
\label{cor_n_alpha}
Assume that for $n \geq 1$, $\lambda_n = n^\alpha$ for some $\alpha > 0$. Then the regret bound for the stochastically constrained adversarial regime with a unique best arm $i^*$ in Theorem~\ref{th_block_lambda_n} satisfies
\begin{multline*}
    R\lr{T, \lr{\lambda_n}_{n\ge 1}}\\
   \leq 
   \mathcal{O} \lr{ \sum_{i \neq i^*} \frac{  K^{\frac{2 \alpha +2}{2\alpha + 3}}T^{\frac{2 \alpha +1}{2\alpha + 3}} + K^{\frac{2 \alpha }{2\alpha + 3}}T^{\frac{4 \alpha }{2\alpha + 3}}}{\Delta_i}}.
\end{multline*}
\end{corollary}
% The proof is given in \Cref{appen_varying_switching_cost}. 
When taking the limit $\alpha\to 0$, this bound scales as $\mathcal{O}\big(K^{2/3}T^{1/3} \sum_{i \neq i^*} \frac{1}{\Delta_i}\big)$, which matches the pseudo-regret bound in the stochastically constrained adversarial regime of Theorem \ref{th_param_switches} with $\lambda = 1$.
Note also that the bound remains sublinear in $T$ as long as $\alpha < \frac{3}{2}$. In other words, with a switching cost as high as $\lambda_n = n^{3/2-\varepsilon}$, for any $\varepsilon > 0$, \TsallisB\ has still a sublinear regret.

%TODO give example?
\section{Proofs}
\label{sec:proofs}
%\subsection{Proof Tools}
%
We start by introducing some preliminary definitions and results.
Recall that the pseudo-regret can be decomposed into a sum of stability and penalty terms \citep{LS19, ZS21}.
Let $\Phi_n$ be defined as:
\begin{equation*}
\Phi_n(C) = \max_{p \in \Delta^{K-1}} \lrc{\lrscal{p, C} + \sum_i \frac{4 \sqrt{p_i} - 2 p_i}{\eta_n}}.
\end{equation*}
Note that the distribution $p_n$ used by \TsallisB\ to draw action $I_n$ for block $B_n$ satisfies $p_n = \nabla\Phi_n(-\tilde{C}_{n-1})$. We can write:
\begin{equation}
\label{R_stab_pen}
\begin{aligned}
\mathbb{E}&\left[\sum_{n = 1}^{N} c_{n, I_n}\right] - \min_j\mathbb{E}\left[\sum_{n = 1}^{N} c_{n, j}\right]\\
&\leq  \underbrace{\mathbb{E} \left[ \sum_{n = 1}^{N} c_{n, I_n} + \Phi_n (-\tilde C_n ) - \Phi_n (- \tilde C_{n-1})\right]}_{\text{stability}} \\
&\quad + \underbrace{\mathbb{E} \left[ \sum_{n = 1}^{N}\Phi_n (- \tilde C_{n-1}) - \Phi_n (-\tilde C_n ) - c_{n, i^*_N}\right]}_{\text{penalty}},
\end{aligned}
\end{equation}
where $i^*_N$ is any
% best arm at the end of $N$ blocks.
arm with smallest cumulative loss over the $N$ blocks (i.e., a best arm in hindsight).

We start by introducing bounds on the stability and the penalty parts of the regret. The results generalize the corresponding results of \citet{ZS21} to handle losses that take values in varying ranges and may be larger than 1. The proofs are provided in \Cref{appen:proof_stab_pen}.
Note the multiplicative factor $b_n^2$ in the stability term.
\begin{lemma}
	\label{lem_stab} For any sequence of positive learning rates $\lr{\eta_n}_{n\ge 1}$ and any sequence of bounds $(b_n)_{n\ge 1}$ on the losses at round $n$, the \emph{stability} term of the regret bound of \TsallisB\ satisfies:
	\begin{multline*}
		    \E[ \sum_{n = 1}^N c_{n, I_n} + \Phi_n (-\tilde C_n ) - \Phi_n (- \tilde C_{n-1})] \\
		    \leq \sum_{n = 1}^N \frac{\eta_n}{2}b_n^2 \sum_{i = 1}^K \sqrt{\E[p_{n, i}]}.
		\end{multline*}
	Furthermore, if $ \eta_n b_n \leq \frac{1}{4}$, then for any fixed $j$:
		\begin{multline*}
		\mathbb{E} \left[ c_{n, I_n} + \Phi_n (- \tilde C_{n}) - \Phi_n (-\tilde C_{n-1} )\right]  \\
		\leq   \frac{\eta_n}{2} b_n^2  \sum_{i \neq j} \lr{\sqrt{\mathbb{E} \left[ p_{n, i}\right]} + 2.5 \mathbb{E} \left[ p_{n, i}\right]}.
		\end{multline*} 
		In particular, if there exists $N_0$ such that for all $ n \geq N_0$, $ \eta_n b_n \leq \frac{1}{4}$, then:
        \begin{align*}
        &\E[ \sum_{n = 1}^N c_{n, I_n} + \Phi_n (-\tilde C_n ) - \Phi_n (- \tilde C_{n-1})] \\
        & \leq \sum_{n = 1}^N \frac{\eta_n}{2} b_n^2  \sum_{i \neq j} \lr{\sqrt{\mathbb{E} \left[ p_{n, i}\right]} + 2.5 \mathbb{E} \left[ p_{n, i}\right]} 
        +  \sum_{n = 1}^{N_0} \frac{\eta_n}{2} b_n^2.
        \end{align*}
\end{lemma}
The penalty term is not affected by the change of the range of the losses.
\begin{lemma}
	\label{lem_pen}
	For any non-increasing positive learning rate sequence $\lr{\eta_n}_{n\ge 1}$, the \emph{penalty} term of the regret bound of \TsallisB\ satisfies:
\begin{align*}
		    \mathbb{E}\!\left[ \sum_{n = 1}^N \Phi_n (- \tilde C_{n-1}) - \Phi_n (-\tilde C_n ) - c_{n, i^*_N}\right]
		\leq \frac{4\sqrt K}{\eta_N} + 1. 
		\end{align*}
		Furthemore, if we define $\eta_0$, such that $\eta_0^{-1} = 0$, then
		\begin{align*}
		&\mathbb{E}\left[ \sum_{n = 1}^N \Phi_n(- \tilde C_{n-1}) - \Phi_n (-\tilde C_n ) - c_{n, i^*_N}\right] \\
		&\leq
		4 \sum_{n = 1}^N (\eta_n^{-1} - \eta_{n-1}^{-1}) \sum_{i \neq i^*_N} \lr{\sqrt{\E[p_{n, i}]} - \frac{1}{2}\E[p_{n, i}]}
		 + 1.
		\end{align*}
\end{lemma}
%
% The two following lemmas are key to our analysis. 
%% stochastic self bounded bound?
%
We also present a bound for the cumulative switching cost, which is the key to obtain refined guarantees in the stochastically constrained adversarial regime.
\begin{lemma}
\label{lem_bound_switches}
Consider a sequence of switching costs $\lr{\lambda_n}_{n\ge 1}$. Then for any fixed $j$, the cumulative switching cost satisfies
\[ S\lr{T, \lr{\lambda_n}_{n\ge 1}} \leq \lambda_1 + \sum_{n = 1}^N \lr{\lambda_n + \lambda_{n+1}} \sum_{i \neq j} \Pr(I_{n} = i) . \]
\end{lemma}
\begin{proof}[Proof of Lemma \ref{lem_bound_switches}]
By convention, there is always a switch at round $1$. For subsequent rounds, when there is a switch at round $n$ at least one of $I_{n - 1}$ or $I_n$ is not equal to $j$. Thus, we have:
\[ \Pr(I_{n - 1} \neq I_n) \leq \sum_{i \neq j} \Pr(I_{n-1} = i)  + \Pr(I_{n} = i), \]
and the cumulative switching cost satisfies
\begin{align*}
     & S\lr{T, \lr{\lambda_n}_{n\ge 1}} = \lambda_1 + \sum_{n = 2}^N \lambda_n \Pr(I_{n - 1} \neq I_n) \\
     &\quad\qquad \leq \lambda_1 + \sum_{n = 2}^N \lambda_n \lr{ \sum_{i \neq j} \Pr(I_{n-1} = i)  + \Pr(I_{n} = i)} \\
     &\quad\qquad \leq \lambda_1 + \sum_{n = 1}^N  \sum_{i \neq j} \lr{\lambda_n + \lambda_{n+1}} \Pr(I_{n} = i),
\end{align*}
which concludes the proof.
\end{proof}
%
% \subsection{Proof of \Cref{th_param_switches}}
Armed with these results, we can move on to the proof of \Cref{th_param_switches}.
\begin{proof}[Proof of \Cref{th_param_switches}]
%
%%% |B_n| is an appropriate choice of blocks
In order to apply our results to blocks, we first calculate an upper bound on the number of blocks $N$.
The length of the $n$-th block is defined as $|B_n| = \max \Big\{\Big\lceil\frac{3\lambda \sqrt{n}}{2\sqrt K}\Big\rceil, 1 \Big\}$. The sequence $\lr{B_n}_{n\ge 1}$ satisfies $|B_n| \geq b(n)$ for $b(n) = \frac{3\lambda \sqrt{n}}{2\sqrt K}$ and is non-decreasing.
Let $N^* = K^{1/3}(T/\lambda)^{2/3}$ and observe that:
%\begin{equation*}
%    \int_n^{N^*} \frac{3\lambda \sqrt{n}}{2\sqrt K} dn  = \frac{\lambda}{\sqrt{K}}(N^*)^{3/2} = T.
%\end{equation*}
%Therefore,
%\begin{align*}
%    \sum_{n = 1}^{\floor{N^*} + 1} |B_n| 
%    & \geq \sum_{n = 1}^{\floor{N^*} + 1} \frac{3\lambda \sqrt{n}}{2\sqrt K}  \\
%    & \geq \int_{0}^{\floor{N^*} + 1} \frac{3\lambda \sqrt{n}}{2\sqrt K} \\
%    & \geq \int_{0}^{\floor{N^*}} \frac{3\lambda \sqrt{n}}{2\sqrt K} \\
%    & \geq T.
%\end{align*}
\begin{align*}
    \sum_{n = 1}^{\floor{N^*} + 1} |B_n| 
&\ge
    \sum_{n = 1}^{\floor{N^*} + 1} \frac{3\lambda \sqrt{n}}{2\sqrt K}
\ge
    \int_{0}^{\floor{N^*} + 1} \frac{3\lambda \sqrt{n}}{2\sqrt K}
\\&\ge
    \int_{0}^{N^*} \frac{3\lambda \sqrt{n}}{2\sqrt K}
=
    \frac{\lambda}{\sqrt{K}}(N^*)^{3/2}
\ge
    T.
\end{align*}
Thus, we can upper bound $N$ by $K^{1/3}(T/\lambda)^{2/3} + 1$.
% In the rest of the proof we use the decomposition of the pseudo-regret into the stability and penalty terms (\Cref{R_stab_pen}).
% 
%%%  Proof Adv
%\paragraph{The First Statement of the Theorem}
\paragraph{Proof of the adversarial bound.}
We start by focusing on the bound in the adversarial regime. To do so, we need to control the stability and penalty terms in \eqref{R_stab_pen}, and also the number of switches.
As we already said, the number of switches is bounded by the number of blocks,
%\begin{equation*}
%    S_T \leq N = \frac{K^{1/3}T^{2/3}}{\lambda^{2/3}},
%\end{equation*}
$S_T \le N \le K^{1/3}(T/\lambda)^{2/3} + 1$,
and thus the cumulative switching cost satisfies 
%\begin{equation*}
%   \lambda S_T \leq \lambda^{1/3}K^{1/3}T^{2/3}.
%\end{equation*}
$\lambda S_T \le \lambda N \le K^{1/3}T^{2/3}\lambda^{1/3} + \lambda$.

Next, we bound the quantity $\eta_n|B_n|^2$ for all $n \leq N$:
\begin{align}
    \frac{\eta_n}{2}|B_n|^2 \leq \frac{\sqrt{2}}{\sqrt{n}} \lr{\frac{3\lambda \sqrt{n}}{2\sqrt K} +1} \leq \frac{3\lambda}{\sqrt{2K}}  + \frac{\sqrt{2}}{\sqrt{n}}. \label{bound_eta_B2}
\end{align}
Note that even though the last block $B_N$ may be truncated, we can upper bound its length by the non-truncated length of that block. 

Then, we bound the inverse of the learning rate at round $N$, 
\[ \frac{1}{\eta_{N}} \leq \frac{\sqrt{N}}{2\sqrt{2}}\lr{\frac{3\lambda \sqrt{N}}{2\sqrt K} +1} \leq \frac{3\sqrt{2}}{8} \frac{\lambda N}{\sqrt{K}} +  \frac{\sqrt 2}{4} \sqrt{N}. \]

In order to bound the pseudo-regret over the $N$ blocks, we apply inequality \eqref{eq:wo-adv} from Theorem~\ref{th_wo_switching}. We then add the cumulative switching cost and use the upper bound on $N$ derived earlier,
\begin{align*}
     R(T, \lambda)&\leq 3\sqrt{2} \lambda N + 3\sqrt{2KN} + \lambda N + 1 \\
    &= (3\sqrt{2}+1)\lambda N + 3\sqrt{2KN} + 1\\
    & \leq 5.25 \lambda^{1/3}K^{1/3}T^{2/3} + 3\sqrt{2} \frac{K^{2/3}T^{1/3}}{\lambda^{1/3}} \\
    &\quad+ 3\sqrt{2K}+ 5.25\lambda + 6.25. 
\end{align*}
For small $\lambda$ the term $K^{2/3}(T/\lambda)^{1/3}$ dominates the expression.
However, when $\lambda \leq \frac{2}{3}\sqrt{\frac{K}{T}}$, then for all $n \leq T$ we have $\frac{3\lambda \sqrt{n}}{2\sqrt K} \leq \sqrt{\frac{n}{T}} \leq 1 $, which means that $|B_n| = 1$. In this case the algorithm is not using blocks and we have
$\lambda S_T \leq \lambda T \leq \frac{2}{3}\sqrt{KT}$. As we also have $a_n \leq 1$, we get $\frac{\sqrt{2}}{\sqrt{n}} \leq \eta_n \leq \frac{2 \sqrt 2}{\sqrt n}$. In this case we use Lemmas~ \ref{lem_stab} and \ref{lem_pen} to bound the stability and the penalty terms and obtain that stability and penalty are both bounded by $2\sqrt{2KN}$. Thus, overall, for $\lambda \leq \frac{2}{3}\sqrt{\frac{K}{T}}$ we have $R(T, \lambda) \leq 6.4 \sqrt{KT}$,
and for $\lambda \geq \frac{2}{3}\sqrt{\frac{K}{T}}$ we have $K^{2/3}(T/\lambda)^{1/3} \leq 1.15 \sqrt{KT}$.

Piecing together all parts of the bound finishes the proof.
%
%%% Proof Stoch
%~\paragraph{The Second Statement of the Theorem}
\paragraph{Proof of the stochastically constrained adversarial bound.}
We now derive refined guarantees in the stochastically constrained adversarial regime with a unique best arm $i^*$. We start by deriving bounds for the stability and penalty terms in \eqref{R_stab_pen}. 

Let $N_0$ be a constant, such that for $n \geq N_0$ we have $\eta_n |B_n| \leq \frac{1}{4}$. 
We note that $\eta_n |B_n|  \leq \frac{2\sqrt 2}{\sqrt n}$, so picking $N_0 = 128$ works.
%For the stability term we use the first part of Lemma~\ref{lem_stab} for the first $N_0$ rounds and the second part of the lemma with $j = i^*$ afterwards. Using the rightmost expression in \eqref{bound_eta_B2}, and the previously derived bound for $\frac{\eta_n}{2}|B_n|^2$, gives that the stability term in \eqref{R_stab_pen} is upper bounded by
For the stability term we use the second part of Lemma~\ref{lem_stab} with $j = i^*$. Using \eqref{bound_eta_B2} to bound $\frac{\eta_n}{2}|B_n|^2$ we obtain that the stability term is upper bounded by
\begin{align*}
     \sum_{n = 1}^{N}&\lr{\frac{3\sqrt{2} \lambda}{2\sqrt{K}}  + \frac{\sqrt{2}}{\sqrt{n}}} \sum_{i \neq i^*} \lr{\sqrt{\mathbb{E} \left[ p_{n, i}\right]} + 2.5 \mathbb{E} \left[ p_{n, i}\right]} \\
        &\qquad\qquad\qquad\qquad\qquad +  \sum_{n = 1}^{N_0} \ \lr{\frac{3\sqrt{2}}{2} \frac{\lambda}{\sqrt{K}}  + \frac{\sqrt{2}}{\sqrt{n}}}. 
        \end{align*}
For the penalty term, we first bound the difference between the inverse of two consecutive learning rates.
\begin{align*}
     \eta_n^{-1} &- \eta_{n-1}^{-1}\\
%    & = \frac{\sqrt{n}\lr{\frac{3\lambda \sqrt{n}}{2\sqrt K} +1}}{2\sqrt{2}} - \frac{\sqrt{n-1}\lr{\frac{3\lambda \sqrt{n-1}}{2\sqrt K} +1}}{2\sqrt{2}} \\
    & = \lr{\frac{3\lambda \sqrt{n}}{2\sqrt K} + 1} \frac{\sqrt{n}}{2\sqrt{2}}  - \lr{\frac{3\lambda \sqrt{n-1}}{2\sqrt K}+1} \frac{\sqrt{n-1}}{2\sqrt{2}} \\
    & = \frac{3\sqrt 2\lambda}{8\sqrt K} + \frac{\sqrt{n} - \sqrt{n-1}}{2\sqrt{2}}\\
    &\leq \frac{3\sqrt 2\lambda}{8\sqrt K} + \frac{\sqrt 2}{4\sqrt{n}}.
\end{align*}
Now we use the second part of Lemma~\ref{lem_pen} to bound the penalty term as follows
	\begin{align*}
%	    &\text{pen} \leq  \\
%	    & 
		\sum_{n = 1}^{N}\lr{\frac{3\sqrt 2\lambda}{2\sqrt K} + \frac{\sqrt 2}{\sqrt{n}}} \sum_{i \neq i^*}  \lr{\sqrt{\E[p_{n, i}]} - \frac{1}{2}\E[p_{n, i}]}
		 + 1.
	\end{align*}
Summing the two bounds, and using that for all $n, i$, $\E[p_{n, i}] \leq \sqrt{\E[p_{n, i}]}$, we have:
\begin{multline*}
    R_T\leq 
    \sum_{n = 1}^{N} \lr{\lr{\frac{6\sqrt{2}\lambda}{\sqrt{K}}  + \frac{4\sqrt{2}}{\sqrt{n}}} \sum_{i \neq i^*} \sqrt{\E[p_{n, i}]}}\\
+  \frac{3\sqrt{2}\lambda}{2\sqrt{K}} N_0  + 2\sqrt{2N_0} + 1. 
\end{multline*}
Now we use the self-bounding technique \citep{ZS21}, which states that if $L$ and $U$ are such that $L \leq R \leq U$, then $R \leq 2U - L$. For the lower bound $L$, we use the following identity for the regret
\[
    R_T = \sum_{n = 1}^{N} |B_n|\sum_{i \neq i^*} \Delta_i \E[p_{n, i}],
\]
where $B_N$ is truncated, so that $|B_1|+\cdots+|B_N| = T$. Using the previous expression for the upper bound $U$, we get: 
\begin{align*}
    R_T\leq & \sum_{n = 1}^{N} \lr{\frac{12\sqrt{2}\lambda}{\sqrt{K}}  + \frac{8\sqrt{2}}{\sqrt{n}}} \sum_{i \neq i^*} \sqrt{\E[p_{n, i}]}\\
    & - \sum_{n = 1}^{N} |B_n|\sum_{i \neq i^*} \Delta_i \E[p_{n, i}]
+  \frac{544\lambda}{\sqrt{K}}  +66.
\end{align*}
We bound the cumulative switching cost using Lemma \ref{lem_bound_switches}:
\[ \lambda S_T \leq \lambda + \sum_{n = 1}^N \sum_{i \neq i^*} 2\lambda \E[p_{n, i}].\]
We add those two bounds together to obtain a bound on the regret with switching costs. Note (again) that 
% $\sqrt{\E[p_{n, i}]} \leq \E[p_{n, i}]$,
$\E[p_{n, i}] \leq \sqrt{\E[p_{n, i}]}$ for all $n$ and $i$,
and that $\frac{\sqrt{2}}{\sqrt K} \leq 1$. Thus, we can upper bound the pseudo-regret with switching costs as:
\begin{align*}
    &R(T, \lambda)\\
&\le
    \sum_{n = 1}^{N}  \sum_{i \neq i^*} \Bigg(\left(14\lambda  + \frac{8\sqrt{2}}{\sqrt{n}}\right)  \sqrt{\E[p_{n, i}]}
 -
    \Delta_i |B_n|\E[p_{n, i}] \Bigg)\\
    &\quad+ 
    \frac{544\lambda}{\sqrt{K}} + \lambda +66.
\end{align*}
Now we note that each term in the inner sum is an expression of the form $a \sqrt{x} - b x$, which for $x \in [0,\infty]$ is maximized at $x = \frac{a^2}{4b}$. Put attention that the cumulative switching cost is part of the optimization problem.
So, for any $i$ and any $n < N$, we have:
\begin{align}
\nonumber
     &\lr{14\lambda  + \frac{8\sqrt{2}}{\sqrt{n}}} \sqrt{\E[p_{n, i}]} - \Delta_i |B_n|\E[p_{n, i}] \\
\nonumber
    &\qquad \leq  \frac{\lr{14\lambda  + \frac{8\sqrt{2}}{\sqrt{n}}}^2}{4 \Delta_i |B_n|} \\
\label{eq:some-terms}
    &\qquad \leq  \frac{(14\lambda)^2}{4 \Delta_i \lr{\frac{3\lambda \sqrt{n}}{2\sqrt K}}} + 2 \frac{14\lambda \lr{\frac{8\sqrt{2}}{\sqrt{n}}}}{4 \Delta_i} + \frac{\lr{\frac{8\sqrt{2}}{\sqrt{n}}}^2}{4\Delta_i} \\
\label{eq:3-terms}
&\qquad \leq  \frac{33\lambda \sqrt{K}}{\Delta_i \sqrt{n}} + \frac{80\lambda}{\Delta_i \sqrt{n}} + \frac{32}{\Delta_i n},
\end{align}
where in the first term of \eqref{eq:some-terms} we have lower bounded $|B_n|$ by $b_n$ and in the last two terms by 1.
As the last block may be truncated, for $n = N$ we bound $|B_N|$ in the first term in \eqref{eq:3-terms} by $1$, leading to
\begin{multline*}
    %& 
    \lr{14\lambda  + \frac{8\sqrt{2}}{\sqrt{N}}} \left(\sqrt{\E[p_{N, i}]} - \Delta_i |B_N|\E[p_{N, i}] \right) \\
    %& 
    \leq  \frac{49\lambda^2}{\Delta_i} + \frac{80\lambda}{\Delta_i \sqrt{n}} + \frac{32}{\Delta_i n},
\end{multline*}
All that remains is to sum over $n$. For the first term in \eqref{eq:3-terms} we have:
\begin{align*}
 \frac{49\lambda^2}{\Delta_i} + \sum_{n = 1}^{N-1} \frac{33\lambda \sqrt{K}}{\Delta_i \sqrt{n}}
 & \leq 66\frac{\lambda \sqrt{K(N-1)}}{\Delta_i} +  \frac{49\lambda^2}{\Delta_i} \\
 & \leq  66\frac{\lambda^{2/3}T^{1/3}K^{2/3}}{\Delta_i} + \frac{49\lambda^2}{\Delta_i}.
\end{align*}
Similarly, the second term in \eqref{eq:3-terms} gives:
\[
\sum_{n = 1}^{N} \frac{80\lambda}{\Delta_i \sqrt{n}} \leq 160 \frac{\lambda \sqrt{N}}{\Delta_i} \leq  160\frac{\lambda^{2/3}T^{1/3}K^{1/6} + \lambda}{\Delta_i}.
\]
For the last term in \eqref{eq:3-terms}, we use the fact that $N \leq T$ and we have:
\[
\sum_{n = 1}^{N} \frac{32}{\Delta_i n} \leq \frac{32 \ln T}{\Delta_i} + \frac{32}{\Delta_i}.
\]
Putting everything together finishes the proof.
\end{proof}

\section{Experiments}
We compare the performance of \TsallisB \ to different baselines, both in the stochastic and in the stochastically constrained adversarial regime. 
We compare \TsallisB\ with block lengths chosen as in Theorem \ref{th_param_switches} against Tsallis-INF without blocks, and against the BaSE algorithm of \citet{NIPS2019_8341}, which achieves a regret of $\mathcal{O} \lr{\sum_{i \neq i^*} \frac{\log T}{\Delta_i}}$ with $\mathcal{O}\lr{\log T}$ switches in the stochastic regime. We use $T$ to tune the parameters of BaSE, and we consider both arithmetic and geometric blocks ---see \citep{NIPS2019_8341} for details.

We also include in our baselines the EXP3 algorithm with a time-varying learning rate, and the block version of EXP3, where the blocks have length $\lambda^{2/3}\frac{T^{1/3}}{K^{1/3}}$. Both block length and learning rate are chosen according to the analysis of EXP3 in the adversarial regime.

%\begin{figure}[h!]
%  \centering
%    \includegraphics[width=0.6\linewidth]{images/legend.png}
%  \caption{In the following experiments, we use this legend.}
%  \label{fig:legend}
%\end{figure}

In the experiments, we fix the number of arms $K = 8$, and set the expected loss of a suboptimal arm to $0.5$. We generate binary losses using two sets of parameters: an ``easy'' setting, where the gaps $\Delta = 0.2$ are large and the switching costs $\lambda = 0.025$ are small. A ``hard'' setting, where the gaps $\Delta = 0.05$ are small and the switching costs $\lambda = 1$ are large. For each experiment, we plot the pseudo-regret, the number of switches, and the pseudo-regret with switching cost. This allows us to observe the trade-off between the pseudo-regret and the number of switches.

 In the first experiment (Figure \ref{fig:exp1_stoch_data_easy}) we use stochastic i.i.d.\ data with the easy setting ($\Delta = 0.2$ and $\lambda = 0.025$). As the gaps are large, even the methods that do not use blocks are not making many switches, and the best performance is achieved by Tsallis-INF without blocks. In Figure \ref{fig:exp1_stoch_data_hard} we use the hard setting ($\Delta = 0.05$ and $\lambda = 1$). In this case, we see a trade-off between achieving a small pseudo-regret and limiting the cumulative switching cost. The small value of $\Delta$ forces a larger number of switches, and because the cost of switching is now large, the cumulative switching cost dominates the pseudo-regret with switching cost.

In Figure \ref{fig:exp2_lambda=0}, we test a stochastic setting with small gaps and zero switching cost. In this case, we observe that Tsallis-Inf and \TsallisB \ outperform both EXP3 and the BaSE algorithms. Note that here \TsallisB\ and Tsallis-Inf have very similar performances, though not identical due to a slight difference in the tuning of learning rates. 

We present a wider range of experiments in Appendix~\ref{appen:experiments}. We show that our algorithm outperforms the BaSE algorithm in the  stochastically constrained adversarial regime. Being an elimination-based algorithm, BaSE also fails in the adversarial regime. %Furthermore, we show that despite the fact that our refined bounds in the stochastically constrained adversarial regime hold when there is a unique best arm, \TsallisB\  exhibits some control on the number of switches in setting with two best arms.  

\begin{figure}[h!]
  \centering
    \includegraphics[width=0.8\linewidth]{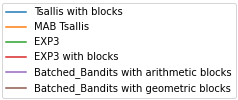}
    \caption{Legend for all plots.}
    \vspace*{5mm}
    \includegraphics[trim={0 0.2cm 0 0},clip,width=\linewidth]{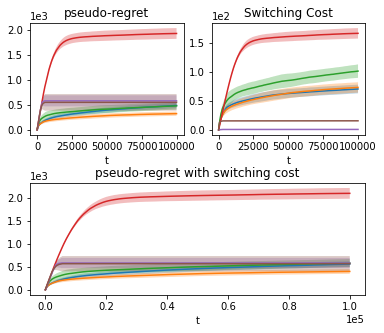}
    \vspace*{-7mm}
    \caption{Stochastic losses, $\Delta = 0.2$ and $\lambda = 0.025$ (easy setting).}
      \label{fig:exp1_stoch_data_easy}
     \vspace*{5mm}
    \includegraphics[trim={0 0.2cm 0 0},clip, width=\linewidth]{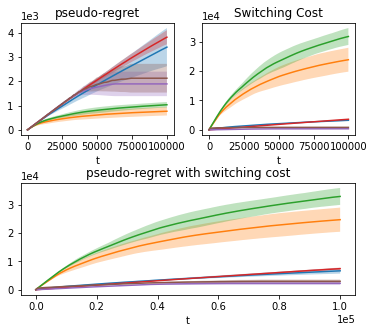}
    \vspace*{-7mm}
  \caption{Stochastic losses, $\Delta = 0.05$ and $\lambda = 1$ (hard setting).}
  \label{fig:exp1_stoch_data_hard}
\end{figure}

\begin{figure}[h!]
  \centering
    \includegraphics[trim={0 5.87cm 0 0},clip, width=\linewidth, ]{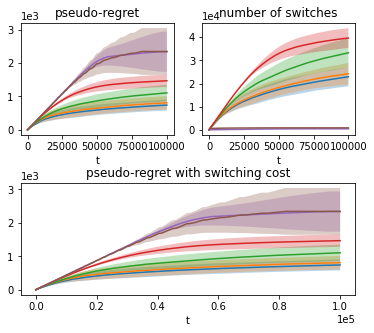}
    \vspace*{-7mm}
  \caption{Stochastic losses and no switching cost, $\lambda = 0$ and $\Delta = 0.05$. As the switching costs are $0$, the pseudo-regret and the pseudo-regret with switching costs are equal.}
  \label{fig:exp2_lambda=0}
\end{figure}
% \begin{figure}[h!]
%   \centering
%     \includegraphics[trim={0 0.2cm 0 0},clip,width=\linewidth]{images/exp4_easy0_False_100000_8_0.5_0.2_0.025_102021-02-04131130.png}
%     \caption{Non Stochastic losses, $\Delta = 0.2$ and $\lambda = 0.025$.}
% %     \includegraphics[trim={0 0.2cm 0 0},clip, width=\linewidth]{images/exp4_hard0_False_100000_8_0.5_0.05_1_102021-02-04132023.png}
% %   \caption{Non Stochastic losses, $\Delta = 0.05$ and $\lambda = 1$. }
%   \label{fig:exp4_non_stoch_data}
% \end{figure}

\section{Discussion}
We introduced \TsallisB, the first algorithm for multiarmed bandits with switching costs that provides adversarial pseudo-regret guarantees simultaneously with improved pseudo-regret guarantees in the stochastic regime, as well as the more general stochastically constrained adversarial regime. The adversarial regret bound matches the minimax lower bound within constants, and guarantees $T^{2/3}$ scaling of the regret in time. The stochastic and stochastically constrained adversarial bounds reduce the dependence of the regret on time down to $T^{1/3}$. Our experiments demonstrate that \TsallisB\ is competitive with the relevant benchmarks over a range of settings: in the stochastic setting, it is competitive with state-of-the-art algorithms for stochastic bandits with switching costs, and outperforms state-of-the-art adversarial algorithms. In the adversarial setting, it is competitive with state-of-the-art adversarial algorithms and significantly outperforms the stochastic ones.

Our work opens multiple directions for future research. For example, it is known that in the stochastic setting with switching costs it is possible to achieve logarithmic regret scaling, but it is unknown whether it is achievable simultaneously with the adversarial regret guarantee. It is also unknown whether logarithmic regret scaling is achievable for the more general stochastically constrained adversarial regime with switching costs (even with no simultaneous requirement of an adversarial regret guarantee). Elimination of the assumption on uniqueness of the best arm in the stochastically constrained adversarial regime is another challenging direction to work on. Unfortunately, for now it is unknown how to eliminate this assumption even in the analysis of the Tsallis-INF algorithm for multiarmed bandits without switching costs. But while in the setting without switching costs the assumption has been empirically shown to be an artifact of the analysis having no negative impact on the regret \citep{ZS21}, in the setting with switching costs treating multiple best arms is more challenging, because switching between best arms is costly.

% In other bandit applications, such as dynamic pricing, there is no explicit switching cost, but prices of items are typically expected not to change too often.

% \subsection{Citations and References}

\section*{Acknowledgements}
CR and YS acknowledge partial support by the Independent Research Fund Denmark, grant number 9040-00361B.
NCB is partially supported by the MIUR PRIN grant Algorithms, Games, and Digital Markets (ALGADIMAR) and by the EU Horizon 2020 ICT-48 research and innovation action number 951847, ELISE (European Learning and Intelligent Systems Excellence).


\bibliography{references, ncb}
\bibliographystyle{icml2021}

\onecolumn
\appendix

\section{Properties of the Potential Function}
\label{appen:prop_poten}

We recall several properties of the potential function provided by \citet[Appendix C]{ZS21}, which we use in our proofs. We use $v =  (v_i)_{i=1,\dots,K}$ to denote a column vector $v\in \R^K$ with elements $v_1,\dots,v_K$, and $\operatorname{diag}(v)$ to denote a $K\times K$ matrix with $v_1,\dots,v_K$ on the diagonal and 0 elsewhere. For a positive semidefinite matrix $M$ we use $||\cdot||_M=\sqrt{\langle\cdot, M\cdot\rangle}$ to denote the canonical norm with respect to $M$.
The potential function is defined as
\begin{equation*}
    \Psi_n(p) = - \sum_i \frac{4 \sqrt{p_i} - 2 p_i}{\eta_n}
\end{equation*}
and we have
\begin{equation*}
    \nabla\Psi_n(p) = - \lr{ \frac{2p_i^{-1/2} - 2}{\eta_n}}_{i = 1, \dots, K}
\end{equation*}
and 
\begin{equation*}
    \nabla^2\Psi_n(p) = \text{diag} \lr{\lr{ \frac{p_i^{-3/2}}{\eta_n}}_{i = 1, \dots, K}}.
\end{equation*}
For $C \leq 0$, the convex conjugate and the gradient of the convex conjugate are
\begin{equation}
    \Psi^*_n(C) = \max_p \lrc{\lrscal{p, C} + \sum_i \frac{4\sqrt{p_i} - 2 p_i}{\eta_n}}, \label{psi*}
\end{equation}
\begin{equation}
    \nabla\Psi^*_n(C) = \argmax_p \lrc{ \lrscal{p, C} + \sum_i \frac{4 \sqrt{p_i} - 2 p_i}{\eta_n} }= \lr{ \lr{-\frac{\eta_n}{2} C_i + 1}^{-2}}_{i = 1, \dots, K}. \label{nabla_psi*}
\end{equation}
We use $\Delta^{K-1}$ to denote the probability simplex over $K$ points and $\mathcal I_{\Delta^{K-1}}(x) =  \begin{cases} 0 &\mbox{if } x \in \Delta^{K-1}\\
-\infty& \mbox{otherwise} \end{cases}$.
% If $C$ were positive, this last equation would not necessarily have a solution in $\mathbb R$.
We also use:
\begin{equation*}
\Phi_n(C) = \lr{\Psi_n + \mathcal I_{\Delta^{K-1}}}^*(C) = \max_{p \in \Delta^{K-1}} \lrc{\lrscal{p, C} + \sum_i \frac{4 \sqrt{p_i} - 2 p_i}{\eta_n}},
\end{equation*}
and 
\begin{equation*}
    \nabla\Phi_n(C) = \argmax_{p \in \Delta^{K-1}} \lrc{\lrscal{p, C} + \sum_i \frac{4 \sqrt{p_i} - 2 p_i}{\eta_n}}.
\end{equation*}
$\Phi_n $ is a constrained version of $\Psi^*_n$, where $p$ is restricted to the probability simplex.
Following \citet[Section 4.3]{ZS21}, there exists a Lagrange multiplier $\nu$ such that:
\begin{equation}
    p_n = \nabla \Phi_n(-\Tilde C_{n-1}) = \nabla \Psi^*_n (-\Tilde C_{n-1} + \nu \bm 1_K) \label{pn_lagrange}
\end{equation}

It is important to note that $\Psi_n$ is a Legendre function, which implies that its gradient is invertible and $(\nabla \Psi_n)^{-1} = \Psi_n^*$.
By the Inverse Function theorem 
\begin{equation}
    \nabla^2\Psi_n^*\big(\nabla\Psi_n(w)\big) = \big(\nabla^2\Psi_n(w)\big)^{-1}.
    \label{prop_pot_inv_fun_th}
\end{equation}
The Bregman divergence associated with a Legendre function $f$ is defined by:
\begin{equation}
   D_f(x, y) = f(x) - f(y) - \lrscal{\nabla f(y), x - y}.\label{prop_pot_breg}
\end{equation}
By Taylor's theorem,
\begin{equation}
   D_f(x, y) \leq \frac{1}{2} \| x - y\|^2_{\nabla^2f(z)}.\label{prop_pot_breg_upperbound}
\end{equation}
for some $z \in \text{conv}(x, y)$.

\section{Proofs of the Lemmas}
\label{appen:proof_stab_pen}

Here we provide a proof of the bound on the stability term in Lemma~\ref{lem_stab}. The scaling of the stability term directly depends on the bound on the losses, so we adapt the bound for sequences of losses that are not in the $[0, 1]$ interval. 
Lemma~\ref{lem_pen} follows directly from \citet[Lemma 12]{ZS21}. We focus on the case where $\alpha = \frac{1}{2}$ and in the second part of the lemma we pick $x = -\infty$.

\subsection{Bounding the Stability}
The proof of Lemma \ref{lem_stab} closely follows the proof of the corresponding result by \citet[Lemma 11]{ZS21}. The main adaptation that we make is to take care of the losses that take values in $[0,b_n]$ intervals rather than $[0,1]$ intervals.

In order to prove Lemma $\ref{lem_stab}$, we first need to adapt \citet[Lemma 17]{ZS21} to properly scale with the range $b_n$. Furthermore, we take advantage of the fact that $\alpha = \frac{1}{2}$ in order to derive a tighter multiplicative constant.

\begin{lemma}
\label{lem_aux_stab}
Let $p \in \Delta^{K-1}$ and $\tilde p = \nabla\Psi_n^* (\nabla\Psi_n(p) - c)$. If $\eta_n b_n \leq \frac{1}{4}$ and $\alpha = \frac{1}{2}$, then for all $c \in \mathbb R^K$ with $c_i \leq -b_n$ for all $i$, it holds that $\tilde p_i^{3/2} \leq 1.5 p_i^{3/2}$ for all $i$.
\end{lemma}
Note that we obtain a slightly better constant factor $1.5$ rather than factor $2$ in the more general analysis by \citet[Lemma 17]{ZS21}.
\begin{proof}
Since $\nabla\Psi_n$ is the inverse of $\nabla\Psi_n^*$, we have:

\begin{align*}
  &\nabla\Psi_n(p)_i - \nabla\Psi_n( \tilde p)_i   = c_i \geq - b_n, \\
  &\frac{p_i^{-1/2} - 1}{\frac{1}{2}\eta_n} - \frac{\tilde p_i^{-1/2} - 1}{\frac{1}{2}\eta_n}  \leq b_n, \\ 
  &\frac{p_i^{-1/2} - 1}{(\frac{1}{2}\eta_n} - \frac{\tilde p_i^{-1/2} - 1}{\frac{1}{2}\eta_n}  \leq b_n, \\ 
   & \frac{p_i^{-1/2} - 1}{\frac{1}{2}\eta_n b_n} - \frac{\tilde p_i^{-1/2} - 1}{\frac{1}{2}\eta_n b_n}  \leq 1, \\ 
    &\tilde p_i^{1/2} \leq \frac{p_i^{1/2}}{1 - \eta_n b_n\frac{1}{2} p_i^{1/2}}\leq \frac{p_i^{1/2}}{1 - \eta_n b_n\frac{1}{2}}, \\ 
    &\tilde p_i^{3/2}\leq \frac{p_i^{3/2}}{\lr{1 - \frac{1}{2}\eta_n b_n}^{3}}.
\end{align*}
It remains to bound $\lr{1 - \frac{1}{2}\eta_n b_n}^{-3}$. Using the fact that $\eta_n b_n \leq \frac{1}{4}$, we have:
 \[ \lr{1 - \frac{1}{2}\eta_n b_n}^{-3} \leq \lr{1 - \frac{1}{8}}^{-3} \leq \frac{8^3}{7^3} \leq 1.5.\]
\end{proof}
With this Lemma at hand, we can move on to the proof of 
 Lemma \ref{lem_stab}. We first verify that the bound still holds for losses outside of the $[0, 1]$ interval, and then we observe how the bound scales in terms of the bounds $b_n$.
\begin{proof}[Proof of Lemma \ref{lem_stab}]
The beginning of the proof is useful for both statements of the Lemma. 
%% Note: we follow the result from the version 3 of the Tsallis Inf paper. We do not also use an upper bound by 1 on the instantaneous stability as we are not using it in the lemma. 

By definition, we have $p_n = \nabla\Phi_n(- \Tilde C_{n-1})$ and $c_{n, I_n} = \lrscal{p_n, \tilde c_n}$. We also have $\Phi_n(C + x\bm 1_K) = \Phi_n(C) + x$, because 
 \begin{align*}
 \Phi_n(C + x\bm 1_K) &= \max_{p \in \Delta^{K-1}} \lrc{\lrscal{p, C + x\bm 1_K} + \sum_i \frac{4 \sqrt{p_i} - 2 p_i}{\eta_n}} \\
 & =  \max_{p \in \Delta^{K-1}} \lrc{\lrscal{p, C} + \lrscal{p,  x\bm 1_K }  + \sum_i \frac{4 \sqrt{p_i} - 2 p_i}{\eta_n}} \\
 & =  \max_{p \in \Delta^{K-1}} \lrc{\lrscal{p, C} + x  + \sum_i \frac{4 \sqrt{p_i} - 2 p_i}{\eta_n}} = \Phi_n(C) + x.
\end{align*}
Using Equation \ref{pn_lagrange}, there exists a constant $\lambda_n$, such that $\nabla\Psi_n(p_n) = -\Tilde C_{n - 1} + \lambda_n \bm 1_K $. Hence, for any $x \in \mathbb R$:
\begin{align*}
\E[c_{n, I_n} + \Phi_n(-\Tilde C_n) + \Phi_n(-\Tilde C_{n - 1})] 
& = \E[\lrscal{p_n, \tilde c_n} + \Phi_n(-\Tilde C_n) + \Phi_n(-\Tilde C_{n - 1})] \\
& = \E[\lrscal{p_n, \tilde c_n} + \Phi_n(\nabla\Psi_n(p_n) - \tilde c_n) + \Phi_n(\nabla\Psi_n(p_n))] \\
& = \E[\lrscal{p_n, \tilde c_n - x\bm 1_K} + \Phi_n(\nabla\Psi_n(p_n) - \tilde c_n + x\bm 1_K) + \Phi_n(\nabla\Psi_n(p_n))] \\
& \leq \E[\lrscal{p_n, \tilde c_n - x\bm 1_K} + \Psi^*_n(\nabla\Psi_n(p_n) - \tilde c_n + x\bm 1_K) + \Psi^*_n(\nabla\Psi_n(p_n))] \numberthis{} \label{proof_l1_1}\\
& = \E[ D_{\Psi^*_n} \lr{\nabla\Psi_n(p_n) - \tilde c_n + x\bm 1_K, \nabla\Psi_n(p_n)}] \\
& \leq \E[ \max_{z \in \text{conv}(\nabla \Phi_n(p_n), \nabla\Psi_n(p_n) - \tilde c_n + x\bm 1_K)}  \frac{1}{2} \| \tilde c_n - x \bm 1_K \|^2_{\nabla^2\Psi^*_n(z)}] \numberthis{} \label{proof_l1_2} \\
& = \E[ \max_{p\in \text{conv}(p_n, \nabla\Psi^*_n(\nabla\Psi_n(p_n) - \tilde c_n + x\bm 1_K))}  \frac{1}{2} \| \tilde c_n - x \bm 1_K \|^2_{\nabla^2\Psi_n(p)^{-1}}] \numberthis{} \label{proof_l1_3} \\
& \leq \E[ \sum_{i = 1}^K \max_{p\in \lrs{p_{n, i},\nabla\Psi^*_n(\nabla\Psi_n(p_n) - \tilde c_n + x\bm 1_K))_i}}  \frac{\eta_n}{2}\lr{\tilde c_{n, i} - x}^2 p_i^{3/2} ],
\end{align*}
where \Cref{proof_l1_1} uses that $\Phi_n(x) \leq \Psi^*_n(x)$, because $\Phi_n$ is a constrained version of $\Psi^*_n$, and $\Phi_n(\nabla\Psi_n(p_n)) = \Psi^*_n(\nabla\Psi_n(p_n))$, because $\arg\max_{p \in \mathbb R^K} \lrscal{p, \nabla\Psi_n(p_n)} - \Psi_n(p) = p_n$ and $p_n$ is in the probability simplex, so the constraint in $\Phi_n$ is inactive. 
\Cref{proof_l1_2} follows from \Cref{prop_pot_breg_upperbound}, and \Cref{proof_l1_3} from \Cref{prop_pot_inv_fun_th}.

\paragraph{First part of the Lemma} 
In order to prove the first part of the Lemma, we set $x = 0$ and observe that $\nabla \Psi^*_n\lr{\nabla\Psi_n(p_n) - \tilde c_n}_i \leq \nabla \Psi^*_n\lr{\nabla\Psi_n(p_n)}_i = p_{n, i}$, because the losses are non-negative and $\nabla \Psi^*_n(C) = \argmax_p \lrc{ \lrscal{p, C} + \sum_i \frac{4 \sqrt{p_i} - 2 p_i}{\eta_n} } $ is a monotonically increasing function of $C$. This observation implies that the highest value of $\lrs{p_{n, i},\nabla\Psi^*_n(\nabla\Psi_n(p_n) - \tilde c_n + x\bm 1_K))_i}$ is $p_{n, i}$. Since the importance weighted losses are $0$ for the arms that were not played, we have:
 \begin{align*}
      \E[ \sum_{i = 1}^K \max_{p\in \lrs{p_{n, i},\nabla\Psi^*_n(\nabla\Psi_n(p_n) - \tilde c_n + x\bm 1_K))_i}}  \frac{\eta_n}{2}\tilde c_{n, i}^2 p_i^{3/2} ] 
      & =  \E[ \sum_{i = 1}^K  \frac{\eta_n}{2} \tilde c_{n, i}^2 p_{n, i}^{3/2} ] \\
      & =  \E[ \sum_{i = 1}^K  \frac{\eta_n}{2} \frac{c_{n, i}^2}{p_{n, i}^2} \Ind{I_n = i} p_{n, i}^{3/2} ] \\
      & = \E[ \sum_{i = 1}^K  \frac{\eta_n}{2} b_n^2 p_{n, i}^{1/2} ] \\
      & = \frac{\eta_n}{2} b_n^2 \sum_{i = 1}^K  \E[   p_{n, i}]^{1/2}, 
 \end{align*}
 where  we use the fact that $c_{n, i}^2 \leq b_n^2$, and that $\mathbb E_n\big[\Ind{I_n = i}\big] = p_{n, i}$, where $\Ind{I_n=i}$ is the indicator function of the event $\{I_n=i\}$ and the expectation is taken with respect to all randomness prior to round $n$. We use Jensen's inequality in the last inequality. Finally, summing on $n$ finishes this part of the proof.
 
 \paragraph{Second part of the Lemma} 
We now set $x = \mathds 1_n[I_n = j]c_{n, j}$, where $ 1_n[\cdot]$ is conditioned on all randomness previous to block $n$. In the calculation below, for the events $I_n \in [K] \backslash \lrc j$, we have $x = 0$ and use the same derivation as in the previous case. When $I_n = j$, for $i \neq j$ we have $\tilde c_{n, i} - x = - x \geq - b_n$, and for $j$ we have $\tilde c_{n, j} - x \geq 0$.
For $i \neq j$ we use Lemma \ref{lem_aux_stab} to bound $\lr{\nabla\Psi^*_n(\nabla\Psi_n(p_n) - \tilde c_n + x\bm 1_K))_i}^{3/2} \leq 1.5p_{n, i}^{3/2}$ and for $j$ we use $\nabla \Psi^*_n(\nabla\Psi_n(p_n) - \tilde c_n)_j \leq \nabla \Psi^*_n(\nabla\Psi_n(p_n))_j = p_{n, j} $. Therefore, we can write
\begin{align*}
     & \E[ \sum_{i = 1}^K \max_{p\in \lrs{p_{n, i},\nabla\Psi^*_n(\nabla\Psi_n(p_n) - \tilde c_n + x\bm 1_K))_i}}  \frac{\eta_n}{2}\lr{\tilde c_{n, i} - x}^2 p_i^{3/2} ]  \\
    &\hspace{4cm}  = \sum_{i \neq j} \frac{\eta_n b_n^2}{2} \E[   p_{n, i}]^{1/2} + \mathbb E\lrs{ \mathds  1_n[I_n = j](j) \lr{\frac{\eta_n}{2}\lr{\frac{c_{n, j}}{p_{n, j}} - c_{n, j}}^2 p_{n, j}^{3/2} + \sum_{i \neq j} \frac{\eta_n}{2}c_{n, j}^2 1.5p_{n, i}^{3/2} }  } \\
    &\hspace{4cm}  = \sum_{i \neq j} \frac{\eta_n b_n^2}{2} \E[   p_{n, i}]^{1/2} + \E[  \frac{\eta_n b_n^2}{2}\lr{1 - p_{n, j}}^2 p_{n, j}^{1/2} + \sum_{i \neq j} \frac{1.5}{2}\eta_n b_n^2 p_{n, i}^{3/2} p_{n, j} ]\\
    &\hspace{4cm}  = \frac{\eta_n b_n^2}{2} \sum_{i \neq j}  \lr{\E[   p_{n, i}]^{1/2} + 2.5 \E[   p_{n, i}]},
    \end{align*}
where in the last step we used the fact that $\lr{1 - p_{n, j}}^2 p_{n, j}^{1/2} \leq \lr{1 - p_{n, j}} = \sum_{i \neq j} p_{n,i} $ for the middle term and $ p_{n, i}^{1/2} p_{n, j} \leq 1$ for the last term.
\end{proof}

% \subsection{Bounding the Penalty}

% In order to prove Lemma \ref{lem_pen}, we are relying on \citet[Lemma 20]{ZS21}, which is defined as follow:
% \begin{lemma}
% For any $\alpha \in [0, 1]$, any positive learning rate and any fixed $u, v, \in \Delta^{K-1}$, the penalty term satisfies:
% \begin{align*}
%     	& \E[\sum_{n = 1}^N \Phi_n (- \tilde C_{n-1}) - \Phi_n (-\tilde C_n ) - c_{n, i^*_N}] \\
%     	& \leq \E[ \frac{\Psi(v) - \Psi(p_1)}{\eta_1} + \sum_{n = 1}^N \lr{\eta_n^{-1} - \eta_{n-1}^{-1}} \lr{\Psi(v) - \Psi(p_t)} + \frac{\Psi(v) - \Psi(p_t)}{\eta_T} ]
% \end{align*}
% \end{lemma}
% \begin{proof}[Proof of Lemma \ref{lem_pen}]

% \end{proof}

\section{Proof of Theorem \ref{th_wo_switching} and its Corollary} 
A side result of our analysis generalizes the analysis of Tsallis-INF \citep{ZS21} to loss sequences that are not in the $[0, 1]^K$ range.

We start with the proof of Theorem \ref{th_wo_switching}. 

\begin{proof}[Proof of Theorem \ref{th_wo_switching}]
~\paragraph{The Adversarial Regime} The sequence of learning rates $\lr{\eta_t}_{t\ge 1}$ is positive and non decreasing. Therefore, we can apply the first parts of Lemmas \ref{lem_stab} and \ref{lem_pen}, and since $\sum_{i = 1}^K \sqrt{\E[p_{n, i}]}\leq \sqrt K$, we directly obtain the result:
 \[ R_T = \text{stability} + \text{penalty} \leq \sum_{t = 1}^T \frac{\eta_t}{2}b_t^2 \sqrt K + \frac{4\sqrt K}{\eta_T} + 1.\]

\paragraph{The Stochastically Constrained Adversarial Regime}

Now we derive refined guarantees in the stochastically constrained adversarial regime with a unique best arm $i^*$. We start by deriving bounds for the stability and the penalty. 

Let $T_0$ be a constant such that for all $t \geq T_0$ we have $\eta_t b_t \leq \frac{1}{4}$. Then by the last part of Lemma \ref{lem_stab} with $j=i^*$ we have:
\begin{align*}
\text{stab}  \leq &
     \sum_{t = 1}^{T} \frac{\eta_t}{2}b_t^2 \sum_{i \neq i^*} \lr{\sqrt{\mathbb{E} \left[ p_{t, i}\right]} + 2.5 \mathbb{E} \left[ p_{t, i}\right]} +  \sum_{t = 1}^{T_0}  \frac{\eta_t}{2}b_t^2. 
        \end{align*}
For the penalty, we  use the second part of Lemma \ref{lem_pen}:
	\begin{align*}
	    &\text{pen} \leq  
		\sum_{i \neq i^*_T}  \sum_{t = 1}^{T} 4\lr{\eta_t^{-1} - \eta_{t-1}^{-1}} \lr{\sqrt{\E[p_{t, i}]} - \frac{1}{2}\E[p_{t, i}]}
		 + 1.
	\end{align*}
We put the two bounds together and first group the $\sqrt{\E[p_{t, i}]}$ terms and $\E[p_{t, i}]$ terms.
\begin{align*}
    R_T\leq & \sum_{t = 1}^{T} \sum_{i \neq i^*}  \lr{\frac{\eta_t}{2}b_t^2 + 4\lr{\eta_t^{-1} - \eta_{t-1}^{-1}}} \sqrt{\E[p_{t, i}]} + \sum_{t = 1}^{T} \sum_{i \neq i^*}  \lr{\frac{5}{4}\eta_tb_t^2 - 2\lr{\eta_t^{-1} - \eta_{t-1}^{-1}}} \E[p_{t, i}]
    +  \sum_{t = 1}^{T_0}  \frac{\eta_t}{2}b_t^2  + 1.
\end{align*}
If $\frac{5\eta_t}{4}b_t^2 \geq 2 \lr{\eta_t^{-1} - \eta_{t-1}^{-1}}$ for all $t$, then the factor in front of $\E[p_{t, i}]$ is positive and by upper bounding $\E[p_{t, i}]$ by $\sqrt{\E[p_{t, i}]}$ and grouping the first and the second summations we obtain
\[
R_T\leq \sum_{t = 1}^{T} \sum_{i \neq i^*}  \lr{\frac{7}{4}\eta_tb_t^2 + 2\lr{\eta_t^{-1} - \eta_{t-1}^{-1}}} \sqrt{\E[p_{t, i}]} +  \sum_{t = 1}^{T_0}  \frac{\eta_t}{2}b_t^2  + 1.
\]
Otherwise, we upper bound the negative contribution $- 2\lr{\eta_t^{-1} - \eta_{t-1}^{-1}} \E[p_{t, i}]$ by zero and $\E[p_{t, i}]$ by $\sqrt{\E[p_{t, i}]}$ and obtain
\[
R_T\leq \sum_{t = 1}^{T} \sum_{i \neq i^*}  \lr{\frac{7}{4}\eta_tb_t^2 + 4\lr{\eta_t^{-1} - \eta_{t-1}^{-1}}} \sqrt{\E[p_{t, i}]} +  \sum_{t = 1}^{T_0}  \frac{\eta_t}{2}b_t^2  + 1.
\]
Overall, we have
\begin{align*}
    R_T\leq & \sum_{t = 1}^{T} \sum_{i \neq i^*}  \lr{\frac{7}{4}\eta_tb_t^2 + c\lr{\eta_t^{-1} - \eta_{t-1}^{-1}}} \sqrt{\E[p_{t, i}]} +  \sum_{t = 1}^{T_0}  \frac{\eta_t}{2}b_t^2  + 1,
\end{align*}
where
\[
    c = \begin{cases} 2, & \mbox{if~} \frac{5\eta_t}{4}b_t^2 \geq 2 \lr{\eta_t^{-1} - \eta_{t-1}^{-1}} \mbox{~for all $t$},\\4, & \mbox{otherwise.}  \end{cases}
\]

Now we use the self-bounding technique \citep{ZS21}. The self-bounding technique states that if $L$ and $U$ are such that $L \leq R \leq U$, then $R \leq 2U - L$. We use the lower bound stated in the theorem, and the upper bound from the previous expression, and we get: 
\begin{align*}
    R_T\leq &  \sum_{t = 1}^{T} \sum_{i \neq i^*} \lr{\lr{\frac{7}{2}\eta_tb_t^2 + 2c\lr{\eta_t^{-1} - \eta_{t-1}^{-1}}}  \sqrt{\E[p_{t, i}]} - \Delta_i b_t \E[p_{t, i}]}+ \sum_{t = 1}^{T_0}  \eta_t b_t^2  + 2 \\
    \leq &  \sum_{t = 1}^{T} \sum_{i \neq i^*} \frac{\lr{\frac{7}{2}\eta_tb_t^2 + 2c\lr{\eta_t^{-1} - \eta_{t-1}^{-1}}}^2}{4\Delta_i b_t} + \sum_{t = 1}^{T_0}  \eta_t b_t^2  + 2,
\end{align*}
where we used the fact that each term in the first summation is an expression of the form $a \sqrt{x} - b x$, which for $x\geq 0$ is bounded by $\frac{a^2}{4b}$.
\end{proof}
In Corollary~\ref{coro_fixed_B} we consider a special case, where the losses at each round are bounded by a constant $B$.
\begin{proof}[Proof of Corollary \ref{coro_fixed_B}]
%Proof of Corollary 5
The learning rate $\eta_t = \frac{2}{B\sqrt t}$ is a positive and non-increasing sequence, which allows us to use the results of Theorem \ref{th_wo_switching}.
\paragraph{The Adversarial Regime} 
In the adversarial regime, we can directly use the learning rate in the first part of Theorem \ref{th_wo_switching} and get:
\begin{align*}
    R_T \leq \sum_{t = 1}^T \frac{\eta_t}{2}B^2 \sqrt K + \frac{4\sqrt K}{\eta_T} + 1
    \leq 4B \sqrt{KT} +  1.
\end{align*}
\paragraph{The Stochastically Constrained Adversarial Regime} 
In order to use the second part of Theorem \ref{th_wo_switching}, we need to bound the difference between two successive learning rates. 
\[ \eta_{t}^{-1} - \eta_{t-1}^{-1} = \frac{B}{2} \lr{\sqrt{t} - \sqrt{t-1}} \leq \frac{B}{2\sqrt t}. \]
We pick $T_0 = 64$, which satisfies that for all $t \geq T_0$, we have $\eta_t B  = \frac{2}{\sqrt T}\leq \frac{1}{4}$.
We note that \[ \frac{5\eta_t}{4}B^2 - 2 \lr{\eta_t^{-1} - \eta_{t-1}^{-1}}  \geq \frac{5B}{2\sqrt t} - \frac{2B}{2\sqrt t} \geq 0. \]
Thus, we have:
\begin{align*}
    R_T  \leq \sum_{t = 1}^{T} \sum_{i \neq i^*} \frac{81 B}{4\Delta_i t} +  \sqrt{B T_0}  + 2
    \leq  \sum_{i \neq i^*} \frac{21 B \lr{\lr{\ln T}+1}}{\Delta_i} +  8\sqrt{B}  + 2.
\end{align*}
\end{proof}

\section{Proofs of Results with Time-Varying Switching Cost}
\label{appen_varying_switching_cost}

In this regime, the block lengths and the learning rates depend on the sequence of switching costs $\lr{\lambda_n}_{n = 1, 2, \dots}$.
\begin{proof}[Proof of Theorem \ref{th_block_lambda_n}]
The switching costs are positive, which means that the learning rate $\eta_n = \frac{2\sqrt{2}\sqrt{K}}{3\lr{\sum_{s = 1}^n \lambda_s + \frac{\sqrt{K}}{\sqrt{s}}}}$ is positive and non-decreasing. Thus, we can apply Theorem \ref{th_wo_switching} and Lemmas \ref{lem_stab} and \ref{lem_pen} through the rest of the proof. 

We recall that the length of the $n$-{th} block is defined as $|B_n| = \max\left\{\left\lceil \frac{\sqrt{\lambda_n} \sqrt{\sum_{s= 1}^n\lambda_s + \frac{\sqrt{K}}{\sqrt{s}}}}{\sqrt{K}}\right\rceil , 1\right\}$.

Thus, we can bound $\frac{\eta_n}{2}|B_n|^2 $ as:
\begin{align*}
    \frac{\eta_n}{2}|B_n|^2 \leq& \frac{\sqrt{2}\sqrt{K}}{3\lr{\sum_{s = 1}^n \lambda_s + \frac{\sqrt{K}}{\sqrt{s}}}} \lr{ \frac{\sqrt{\lambda_n} \sqrt{\sum_{s= 1}^n\lambda_s + \frac{\sqrt{K}}{\sqrt{s}}}}{\sqrt{K}} + 1}^2\\
   \leq & \frac{\sqrt{2}\sqrt{K}}{3\lr{\sum_{s = 1}^n \lambda_s + \frac{\sqrt{K}}{\sqrt{s}}}} \lr{\frac{\sqrt{\lambda_n} \sqrt{\sum_{s= 1}^n\lambda_s + \frac{\sqrt{K}}{\sqrt{s}}}}{\sqrt{K}}}^2 + \frac{2\sqrt{2}\sqrt{K}}{3\lr{\sum_{s = 1}^n \lambda_s + \frac{\sqrt{K}}{\sqrt{s}}}} \lr{ \frac{\sqrt{\lambda_n} \sqrt{\sum_{s= 1}^n\lambda_s + 
   \frac{\sqrt{K}}{\sqrt{s}}}}{\sqrt{K}}} \\
   & + \frac{\sqrt{2}\sqrt{K}}{3\lr{\sum_{s = 1}^n \lambda_s + \frac{\sqrt{K}}{\sqrt{s}}}} \\
   \leq & \frac{\sqrt{2}\lambda_n}{3\sqrt K}  + \frac{2\sqrt{2}\sqrt{\lambda_n}}{3\lr{K^{1/4}n^{1/4}}}
   + \frac{\sqrt{2}}{3\sqrt{n}} \\
   \leq & \frac{\sqrt{2}\lambda_n}{\sqrt K} + \frac{\sqrt{2}}{\sqrt{n}},
\end{align*}
where we use that $ \frac{1}{\sum_{s= 1}^n\lambda_s + 
   \frac{\sqrt{K}}{\sqrt{s}}} \leq \frac{1}{\sqrt{Kn}}$
and we deduce that
$ \frac{\sqrt{\lambda_n}}{\sqrt{\sum_{s= 1}^n\lambda_s + 
   \frac{\sqrt{K}}{\sqrt{s}}}} \leq \frac{\lambda_n}{\sqrt K}
   + \frac{1}{\sqrt{n}}$ by considering the cases $\lambda_n \geq \frac{\sqrt K}{\sqrt n}$ and $\lambda_n \leq \frac{\sqrt K}{\sqrt n}$.

\paragraph{The Adversarial Regime} 
The weighted switching cost on $N$ blocks is upper bounded by $\sum_{n = 1}^{N} \lambda_n$.
To bound the pseudo-regret, we can directly apply Theorem \ref{th_wo_switching} and get that:
\begin{align*}
    R_T & \leq \sum_{n = 1}^{N} \frac{\eta_n}{2}|B_n|^2 \sqrt K + \frac{4\sqrt K}{\eta_{N}} + 1 \\
    & \leq \sum_{n = 1}^{N} \sqrt{2}\lr{\lambda_n
   + \frac{\sqrt K}{\sqrt{n}}}  + 4\sqrt K \frac{3\lr{\sum_{s = 1}^n \lambda_s + \frac{\sqrt{K}}{\sqrt{s}}}}{2\sqrt{2}\sqrt{K}} + 1 \\
    & \leq 4 \sqrt{2}\sum_{n = 1}^{N} \lambda_n + 8\sqrt{2} \sqrt{KN}
  + 1.
\end{align*}
Combining the pseudo regret and the weighted switching cost finishes this part of the proof.

\paragraph{The Stochastically Constrained Adversarial Regime} 
We start by deriving a bound for the stability term. Let $N_0$ be the smallest number, such that for all $n \geq N_0$, we have $\eta_n |B_n| \leq \frac{1}{4}$. Then, using the last part of Lemma \ref{lem_stab}, we have:
\begin{align*}
    \text{stab} \leq & \sum_{n = 1}^N \lr{\frac{\sqrt{2}\lambda_n}{\sqrt K} + \frac{\sqrt{2}}{\sqrt{n}}} \sum_{i \neq i^*} \lr{ \sqrt{\E[p_{n, i}]} + 2.5 \E[p_{n,i}]} + \sum_{n = 1}^{N_0} \lr{\frac{\sqrt{2}\lambda_n}{\sqrt K} + \frac{\sqrt{2}}{\sqrt{n}}}.
\end{align*}
We now bound the penalty term.
We first need to bound the difference between two successive learning rates.
\begin{align*}
    \eta_{n}^{-1} - \eta_{n-1}^{-1}
    &= \frac{3}{2\sqrt{2}\sqrt{K}} \lr{\sum_{s = 1}^n \lambda_s + \frac{\sqrt{K}}{\sqrt{s}} - \sum_{s = 1}^{n-1} \lambda_s + \frac{\sqrt{K}}{\sqrt{s}}} \\
    &= \frac{3}{2\sqrt{2}\sqrt{K}} \lr{\lambda_n + \frac{\sqrt{K}}{\sqrt{n}}}.
\end{align*}
 
 Then, we apply the second part of Lemma \ref{lem_pen} and get:
 \begin{align*}
     \text{penalty} \leq  \sum_{i \neq i^*} \lr{\frac{3\sqrt{2}}{\sqrt{K}} \lr{\lambda_n + \frac{\sqrt{K}}{\sqrt{n}}}} \lr{ \sqrt{\E[p_{n, i}]} - 0.5\E[p_{n,i}]} + 1.
 \end{align*}
Adding these expressions together and using the self-bounding technique, we have:
\[R_T \leq  \sum_{n = 1}^N 10\sqrt 2\lr{\frac{\lambda_n}{\sqrt K} + \frac{1}{\sqrt{n}}} \sum_{i \neq i^*}  \sqrt{\E[p_{n, i}]} - \sum_{i \neq i^*}  \sum_{n = 1}^N \Delta_i |B_n| \E[p_{n, i}]
+ \sum_{n = 1}^{N_0} \lr{\frac{2\sqrt{2}\lambda_n}{\sqrt K}}
+ 4\sqrt{2N_0} + 2.\]
Finally, we use Lemma \ref{lem_bound_switches} to bound the number of switches, and the fact that $\sqrt{\E[p_{n, i}]} \geq \E[p_{n, i}]$, which gives:
\[ R\lr{T, \lr{\lambda_n}_{n\ge 1}} \leq  \sum_{n = 1}^N\sum_{i \neq i^*}  \lr{\lr{11 \lambda_n  + \lambda_{n + 1} + \frac{10 \sqrt{2}}{\sqrt{n}}}  \sqrt{\E[p_{n, i}]} - \Delta_i |B_n| \E[p_{n, i}]}
+ \sum_{n = 1}^{N_0} \lr{\frac{2\sqrt{2}\lambda_n}{\sqrt K}}
+ 4\sqrt{2N_0} + \lambda_1 + 2. \] 

We then observe that for each term in the first summation, we can upper bound the expression by replacing $\E[p_{n, i}]$ by $x_{n, i} \in [0, \infty)$ and maximizing each term independently on $[0, \infty)$.

Thus, we have:
\[ R\lr{T, \lr{\lambda_n}_{n\ge 1}} \leq  \sum_{n = 1}^N\sum_{i \neq i^*} \frac{ \lr{11 \lambda_n  + \lambda_{n + 1} + \frac{10 \sqrt{2}}{\sqrt{n}}}^2}{4\Delta_i |B_n|}
+ \sum_{n = 1}^{N_0} \lr{\frac{2\sqrt{2}\lambda_n}{\sqrt K}}
+ 4\sqrt{2N_0} + \lambda_1 + 2. \] 
\end{proof}
We move on to the corollary with a parametric form of switching costs.

\begin{proof}[Proof of Corollary \ref{cor_n_alpha}] In this setting, we assume that the sequence of switching costs satisfies $\lambda_n = n^\alpha$ for $\alpha > 0$.
We start by upper bounding $N_0$.
\begin{align*}
    \eta_n |B_n| & \leq  \frac{2\sqrt{2}\sqrt{K}}{3\lr{\sum_{s = 1}^n \lambda_s + \frac{\sqrt{K}}{\sqrt{s}}}} \lr{\frac{\sqrt{\lambda_n} \sqrt{\sum_{s= 1}^n\lambda_s + \frac{\sqrt{K}}{\sqrt{s}}}}{\sqrt{K}} + 1} \\
     &\leq  \frac{2\sqrt{2} \sqrt{\lambda_n}}{3 \sqrt{\sum_{s = 1}^n \lambda_s + \frac{\sqrt{K}}{\sqrt{s}}}} +   \frac{2\sqrt{2}\sqrt{K}}{3\lr{\sum_{s = 1}^n \lambda_s}}\\
     & \leq  \frac{2\sqrt{2} n^{\alpha/2}}{3 \sqrt{\frac{n^{\alpha + 1}}{\alpha + 1}}} +   \frac{2\sqrt{2}}{3 \sqrt{n}} \\
      & \leq  \frac{2\sqrt{2}}{3 \sqrt{n}} \lr{\sqrt{\alpha + 1}+1},
\end{align*}
which is a decreasing sequence of $n$. For all $n \geq \frac{32}{9} \lr{\sqrt{\alpha + 1} + 1}^2$, we have $\eta_n|B_n| \leq \frac{1}{4}$, thus we pick $N_0 = \big\lceil\frac{32}{9} \lr{\sqrt{\alpha + 1} + 1}^2\big\rceil$.

Now we move on to bounding the terms $\frac{ \lr{11 \lambda_n  + \lambda_{n + 1} + \frac{10 \sqrt{2}}{\sqrt{n}}}^2}{4\Delta_i |B_n|}$. Here, the switching costs are increasing, $\lambda_{n + 1} \geq \lambda_n$, and we have:
\begin{align*}
    \frac{ \lr{11 \lambda_n  + \lambda_{n + 1} + \frac{10 \sqrt{2}}{\sqrt{n}}}^2}{4\Delta_i |B_n|}
    \leq \frac{ \lr{12 \lambda_{n + 1} + \frac{10 \sqrt{2}}{\sqrt{n}}}^2}{4\Delta_i |B_n|}
    \leq \frac{ \lr{12^2 \lambda_{n + 1} + 240 \lambda_{n + 1}\frac{\sqrt{2}}{\sqrt n}  + \frac{200}{n}}}{4\Delta_i |B_n|}.
\end{align*}
When $ n < N$ the block has not been truncated and the first term is upper bounded as:
\[ \frac{ 12^2 \lambda_{n+1}^2}{4\Delta_i |B_n|} \leq \frac{ 36 \sqrt{\alpha + 1} \lr{n + 1}^{2\alpha} \sqrt{K}}{\Delta_i n^{\alpha + 1/2}}
    \leq \frac{ 36 \cdot 4^\alpha \sqrt{\alpha + 1} \lr{n}^{\alpha - 1/2} \sqrt{K}}{\Delta_i}, \]
 where  $|B_n| \geq \frac{n^{\alpha + 1/2}}{\sqrt{K}\sqrt{\alpha + 1}}$, and for all $n \geq 1$, we have $\frac{\lr{n + 1}^2}{n} \leq 4n$.
 For the case where $n = N$, we can only lower bound $|B_N|$ by $1$, and we get:
 \[ \frac{ 12^2 \lambda_{N+1}^2}{4\Delta_i |B_N|} \leq \frac{ 36 (N + 1)^{2\alpha}}{\Delta_i}. \]
The second and third terms are directly upper bounded by lower bounding the block length by $1$:
\[ \frac{  240 \frac{\lambda_{n+1}\sqrt{2}}{\sqrt n}}{4 \Delta_i |B_n|} \leq  60 \sqrt{2} \frac{\lr{n+1}^\alpha}{\sqrt n\Delta_i } \leq  60 \sqrt{2} \cdot 2^{\alpha}\frac{\lr{n}^{\alpha-1/2}}{\Delta_i } , \]
and 
\[ \frac{ \frac{200}{n}}{4\Delta_i |B_n|} \leq \frac{50}{n\Delta_i}.
\]
We now sum over $n$, from $1$  to $N-1$, and get:
\begin{align*}
    \sum_{n = 1}^{N-1} n^{\alpha - 1/2} \leq 1 + \int_{1}^N n^{\alpha - 1/2} \leq 1 + \lr{\alpha + 1/2} N^{\alpha + 1/2},
\end{align*}
which is an upper bound which considers the case where $\alpha - \frac{1}{2} \leq 0$ and where $\alpha - \frac{1}{2} \geq 0$. 
We finish the proof by combining these results, and we get:
\begin{align*}
     R\lr{T, \lr{\lambda_n}_{n\ge 1}} 
     \leq  & \sum_{n = 1}^N\sum_{i \neq i^*} \frac{ \lr{11 \lambda_n  + \lambda_{n + 1} + \frac{10 \sqrt{2}}{\sqrt{n}}}^2}{4\Delta_i |B_n|}
+ \sum_{n = 1}^{N_0} \lr{\frac{2\sqrt{2}\lambda_n}{\sqrt K}}
+ 4\sqrt{2N_0} + \lambda_1 + 2 \\
\leq & \sum_{i \neq i^*} \lr{\frac{36 \cdot 4^\alpha \sqrt{\alpha + 1}\lr{\alpha +\frac{1}{2}} N^{\alpha + 1/2} \sqrt K}{\Delta_i} +  {\frac{36 \cdot 4^\alpha \sqrt{\alpha + 1}\sqrt K}{\Delta_i}}} + \frac{36 \lr{N + 1}^{2\alpha}}{\Delta_i} \\ 
& + \sum_{i \neq i^*} \lr{\frac{60\sqrt{2} \cdot 2^{\alpha} \lr{\alpha +\frac{1}{2}} N^{\alpha + 1/2}}{\Delta_i} +  \frac{60\sqrt{2} \cdot 2^{\alpha}}{\Delta_i}} + \frac{60\sqrt{2} \cdot 2^{\alpha} N^{\alpha -1/2}}{\Delta_i} \\
& + \sum_{i \neq i^*} \frac{50 \ln N}{\Delta_i} + \frac{50}{\Delta_i}  + \frac{2\sqrt{2} \lr{\frac{32}{9} \lr{\sqrt{\alpha + 1} + 1}^2+ 2}^{\alpha + 1}}{\lr{\alpha + 1}\sqrt K}
+ 11 \lr{\sqrt{\alpha + 1}} + 20.
\end{align*}
We now upper bound $N$. We first note that the length of the $n$-{th} block is lower bounded by $ \frac{n^{\alpha/2}\sqrt{\sum_{s = 1}^n s^\alpha}}{\sqrt{K}}$.
Using the fact that $\alpha > 0$, we can lower bound ${\sum_{s = 1}^n s^\alpha} \geq \int_{0}^n s^\alpha ds = \frac{n^{\alpha + 1}}{\alpha + 1}$. Thus, we have $|B_n| \geq \frac{n^{\alpha + 1/2}}{\sqrt{(\alpha + 1)K}}$.
Since the block length is an increasing function, for any $\bar N$ we have:
\[ \sum_{n = 1}^{\bar N} |B_n| \geq \int_0^{\bar N} \frac{n^{\alpha + 1/2}}{\sqrt{(\alpha + 1)K}} = \frac{\bar N^{\alpha + 3/2}}{\lr{\alpha + \frac{3}{2}} \sqrt{(\alpha + 1)K}}.\]

We observe that $ \bar N = \lr{\alpha + \frac{3}{2}}^{\frac{2}{2\alpha + 3}}\lr{\alpha + 1}^{\frac{1}{2\alpha + 3}} K^{\frac{1}{2\alpha + 3}} T^{\frac{2}{2\alpha + 3}}$ satisfies $\frac{\bar N^{\alpha + 3/2}}{\lr{\alpha + \frac{3}{2}} \sqrt{(\alpha + 1)K}} = T$.
Thus. we are sure that $N$ is upper bounded by $\lr{\alpha + \frac{3}{2}}^{\frac{2}{2\alpha + 3}}\lr{\alpha + 1}^{\frac{1}{2\alpha + 3}} K^{\frac{1}{2\alpha + 3}} T^{\frac{2}{2\alpha + 3}} + 1.$
All that remains is to upper bound $N$ in the pseudo-regret bound.
\begin{align*}
     R\lr{T, \lr{\lambda_n}_{n\ge 1}} 
\leq & \sum_{i \neq i^*}\frac{36 \cdot 4^\alpha \sqrt{\alpha + 1}\lr{\alpha +\frac{1}{2}} \lr{\alpha + \frac{3}{2}}^{\frac{2\alpha +1}{2\alpha + 3}}\lr{\alpha + 1}^{\frac{\alpha + 1/2}{2\alpha + 3}} T^{\frac{2 \alpha +1}{2\alpha + 3}} K^{\frac{2 \alpha +2}{2\alpha + 3}}}{\Delta_i} \\
& + \sum_{i \neq i^*} {\frac{36 \lr{\alpha + 2}\cdot 4^\alpha \sqrt{\alpha + 1}\sqrt K}{\Delta_i}} + \frac{36 \lr{N + 1}^{2\alpha}}{\Delta_i} \\ 
& + \sum_{i \neq i^*} \frac{60\sqrt{2} \cdot 2^{\alpha} \lr{\alpha +\frac{1}{2}} N^{\alpha + 1/2}}{\Delta_i} +  \frac{60\sqrt{2} \cdot 2^{\alpha}}{\Delta_i}  + \frac{60\sqrt{2} \cdot 2^{\alpha} N^{\alpha -1/2}}{\Delta_i} \\
& + \sum_{i \neq i^*} \frac{50 \ln T}{\Delta_i} + \frac{50}{\Delta_i}  + \frac{2\sqrt{2} \lr{\frac{32}{9} \lr{\sqrt{\alpha + 1} + 1}^2+ 2}^{\alpha + 1}}{\lr{\alpha + 1}\sqrt K}
+ 11 \lr{\sqrt{\alpha + 1}} + 20.
\end{align*}
\end{proof}

\section{Supplementary Experiments}
\label{appen:experiments}
In this section, we present additional experiments highlighting the robustness of \TsallisB. In all the experiments we take $K=8$. Similar results were observed for other values of $K$.
%We keep $K = 8$ for those experiments as well, but the results generalize for other choices of $K$, $\Delta$, and $\lambda$.

First, we consider stochastically constrained adversarial sequences. We take a setting, inspired by \citet{ZS21}, where the environment alternates between two phases. In the first one, the expected loss of the best arm is $0$, and the expected loss of the suboptimal arms is $\Delta$. In the second phase, the expected loss of the best arm is $1 - \Delta$, and the expected loss of suboptimal arms is $1$. At all rounds, the gap between the expected loss of the best arm and any other arm remains constant. In this experiment, the environment generates phases of exponentially increasing length with he $i^{th}$ phase starting at index $1.6^i$.
We observe in Figures~\ref{fig:exp4_non_stoch_easy} and \ref{fig:exp4_non_stoch_hard} that the BaSE algorithm with arithmetic blocks is not robust in this regime. BaSE algorithm with geometric blocks performs really well against this sequence.
% , but changing the lengths of the phases would be sufficient to break BaSE with geometric sequences.
\TsallisB \ performs well in both experiments, achieving a regret with switching costs similar to algorithms without blocks when the switching cost is small, and a much better performance when switching becomes costly.

\begin{figure}[h!]
  \begin{minipage}{\linewidth}
   \centering
    \includegraphics[width = 0.3\linewidth]{images/legend.png}
  \end{minipage}
   \begin{minipage}{0.48\textwidth}
        \centering
       \includegraphics[trim={0 0.2cm 0 0},clip,width=\linewidth]{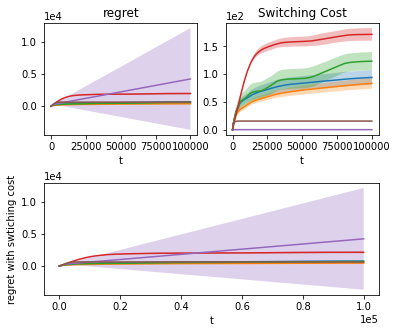}
    \caption{Stochastically constrained adversarial losses, $\Delta = 0.2$ and $\lambda = 0.025$. The shaded area represents one standard deviation above and below the average measured on 10 repetitions of the experiment. The standard deviation of Batched Bandits is large because the algorithm eliminates the optimal arms in some of the runs of the experiment, but not all of them.}
      \label{fig:exp4_non_stoch_easy}
    \end{minipage}\hfill
    \begin{minipage}{0.48\textwidth}
        \centering
        \includegraphics[trim={0 0.2cm 0 0},clip, width=\linewidth]{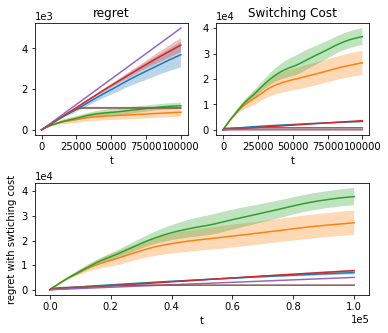}
  \caption{Stochastically constrained adversarial losses. $\Delta = 0.05$ and $\lambda = 1$.  The shaded area represents one standard deviation above and below the average measured on 10 repetitions of the experiment.}
   \label{fig:exp4_non_stoch_hard}
    \end{minipage}
\end{figure}

In the second experiment we construct an adversarial sequence that easily breaks BaSE with both arithmetic and geometric grids. We also observe the behavior of the other algorithms in this context. The sequence of losses is constructed in the following way: in the first $\sqrt{KT \ln(KT)}$ rounds, one arm suffers a loss of 0, while all the other arms suffer a loss of 1. After the $\sqrt{KT \ln(KT)}$ rounds the losses are reversed, so the first arm suffers a loss of $1$ and all other arms suffer a loss of $0$.
In Figure~\ref{fig:exp_adv} we observe that the BaSE algorithm with both arithmetic and geometric grid suffers linear regret, as it, with high probability, eliminates the best arm based on the first rounds. We can see that with this sequence, \TsallisB \ achieves both a very low regret and a low number of switches, even though at the end of the game, $K -1$ arms have the same performance, and only one is suboptimal.

\begin{figure}[h!]
 \centering
    \begin{minipage}{0.48\textwidth}
        \includegraphics[trim={0 0.2cm 0 0},clip, width=\linewidth]{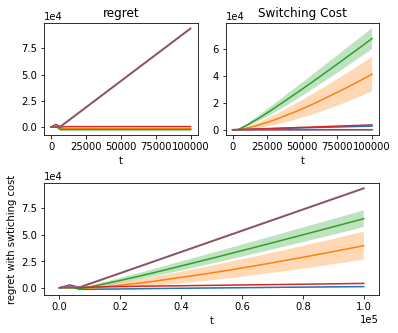}
  \caption{Regret against a deterministic adversarial sequence described in the text with $\lambda = 1$. The shaded area represents one standard deviation above and below the average measured on 10 repetitions of the experiment. The curves for batched bandits with arithmetic blocks and batched bandits with geometric blocks almost coincide and they are the highest ones.}
   \label{fig:exp_adv}
    \end{minipage}
\end{figure}

In the last experiment we test robustness of \TsallisB \ in a stochastic setting with several best arms and a stochastically constrained adversarial setting with several best arms. We take $\Delta = 0.2$ and $\lambda = 1$ and change the number of optimal arms from 1 to 7 while keeping the total number of arms $K = 8$. We recall that \citet{ZS21} experimentally observed that in a stochastic setting without switching costs the regret of Tsallis-INF decreases with the increase of the number of best arms, suggesting that the requirement on uniqueness of the best arm is an artifact of the analysis, rather than a real limitation of the algorithm. In Figures~\ref{fig:multi_stoch_data} and \ref{fig:multi_non_stoch_data} we observe that in the setting with switching costs the picture is different, because switching between best arms is costly. We note that \TsallisB\ still has the adversarial regret guarantee of $\scO\big((\lambda K)^{1/3}T^{2/3} + \sqrt{KT}\big)$ in both settings, so the regret is still under control, but there is a clear increase in the regret as the number of optimal arms grows beyond 1. Therefore, the experiments seem to suggest that the improved regret scaling with $T^{1/3}$ only holds under the assumption on uniqueness of the best arm and elimination of this assumption will require modification of the algorithm.

% We observe in Figures~\ref{fig:exp3_stoch_data_easy} and \ref{fig:exp3_stoch_data_hard} the impact of having two best arms instead of one. In both cases, \TsallisB \ performs relatively well, but the most interesting result comes from comparing the experiment with the results in Figure~\ref{fig:exp1_stoch_data_easy} which has the same setting for $\Delta$ and $\lambda$. In the first experiment, \TsallisB \ and Tsallis-INF had a very similar performance, whereas in Figure \ref{fig:exp3_stoch_data_easy}, the performance of \TsallisB \ is significantly better than both Tsallis-INF and EXP3. 
% This shows that even though we did not derive refined bounds for the stochastic regime with multiple best arm, our algorithm shows some robustness against such problem.

\begin{figure}[h!]
   \begin{minipage}{\linewidth}
   \centering
    \includegraphics[width = 0.25\linewidth]{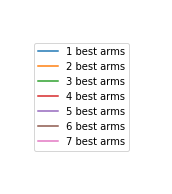}
  \end{minipage}
   \begin{minipage}{0.48\textwidth}
        \centering
       \includegraphics[trim={0 0.2cm 0 0},clip,width=\linewidth]{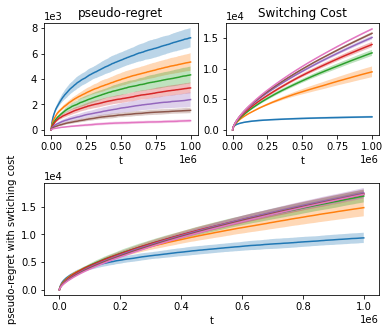}
    \caption{The performance of \TsallisB\ under stochastic losses and several optimal arms. $K = 8$, $\Delta = 0.2$ and $\lambda = 1$. The shaded area represents one standard deviation above and below the average measured on 10 repetitions of the experiment.}
      \label{fig:multi_stoch_data}
    \end{minipage}\hfill
    \begin{minipage}{0.48\textwidth}
        \centering
        \includegraphics[trim={0 0.2cm 0 0},clip, width=\linewidth]{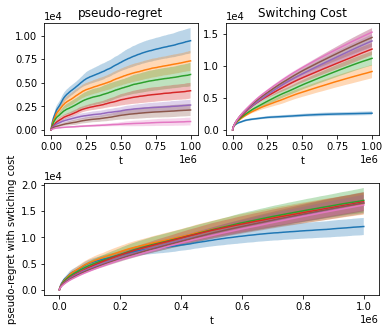}
  \caption{The performance of \TsallisB\ under stochastically constrained adversarial losses and several optimal arms. $K = 8$, $\Delta = 0.2$ and $\lambda = 1$. The shaded area represents one standard deviation above and below the average measured on 10 repetitions of the experiment.}
   \label{fig:multi_non_stoch_data}
    \end{minipage}
\end{figure}

\end{document}